\documentclass[11pt]{article}

\usepackage[preprint]{acl}

\usepackage{times}
\usepackage{latexsym}
\usepackage[T1]{fontenc}
\usepackage[utf8]{inputenc}
\usepackage{microtype}
\usepackage{inconsolata}
\usepackage{graphicx}
\usepackage{booktabs}
\usepackage{amsfonts}
\usepackage{amsmath}
\usepackage{mathtools}
\usepackage{amsthm}
\usepackage{amssymb}
\usepackage{nicefrac}
\usepackage{multirow}
\usepackage{subcaption}
\usepackage{algorithm}
\usepackage{algorithmic}
\usepackage{pifont}
\usepackage[table,xcdraw]{xcolor}
\usepackage[capitalize,noabbrev]{cleveref}
\usepackage{enumitem} 

\newcommand{\cmark}{\ding{51}}
\newcommand{\xmark}{\ding{55}}
\definecolor{darkgreen}{RGB}{0, 100, 0}

\theoremstyle{plain}
\newtheorem{theorem}{Theorem}[section]

\theoremstyle{definition}

\theoremstyle{remark}

\title{ZipRL: Adaptive Multi-Turn Context Compression with Hindsight Response Replay}

\author{
  \textbf{Zhexin Hu}$^{1, 2}$\thanks{Equal contribution} \quad
  \textbf{Li Wang}$^{1}$\footnotemark[1] \quad
  \textbf{Xiaohan Wang}$^{1,*}$\thanks{Corresponding authors} \quad
  \textbf{Jiajun Chai}$^{1}$ \quad
  \textbf{Xiaojun Guo}$^{1}$ \\
  \textbf{Wei Lin}$^{1}$ \quad
  \textbf{Guojun Yin}$^{1}$\footnotemark[2] \\
  \\
  $^{1}$Meituan \\
  $^{2}$Institute of Software, Chinese Academy of Sciences
}

\begin{document}
\maketitle

\begin{abstract}
Adaptive context compression is vital for scaling Large Language Models (LLMs) to complex, multi-turn agent tasks. However, rule-based compression methods may discard task-critical nuances, while Reinforcement Learning (RL) approaches usually struggle to balance information retention and token efficiency under the sparse rewards inherent to long-horizon workflows. To bridge this gap, we propose \textbf{ZipRL}, a novel adaptive compression framework tailored for Reinforcement Learning from Verifiable Rewards (RLVR). ZipRL features a multi-granularity compression mechanism for active, non-uniform information reduction, coupled with Hindsight Response Replay (HRR), a technique designed to densify training signals during RLVR optimization. Theoretically, we prove ZipRL's superior task-relevant utility over uniform methods. Concretely, ZipRL utilizes coarse-to-fine prompts for macro-compression and incorporates HRR into GRPO via generalized advantage reshaping. Multiple models of varying versions and parameter scales validate the effectiveness of our approach. Benchmarks on five agent tasks show ZipRL outperforms state-of-the-art approaches by 27.9\% and 34.7\% across Qwen3-4B and Qwen3-8B models, while maintaining exceptional token efficiency and robustness under extreme 256-turn extrapolation stress tests. Our code is available at \url{https://github.com/huzhexin/ZipRL}.
\end{abstract}

\section{Introduction}

Large language models (LLMs) have emerged as powerful agents for solving long-horizon tasks that require extended reasoning, planning, and interaction with external environments \citep{ke2025survey, wu2024beyond, wu2026atlas}. Effective context management is crucial for unlocking the full potential of these capabilities, as it strategically optimizes the use of the finite context window to ensure coherent, accurate, and efficient model responses \citep{shao2024deepseekmath, song2023llm}. For example, recent research has integrated LLMs with external environments, such as research agents, enabling LLMs to retrieve and process dynamic information for open-ended problem solving~\citep{jin2025search}. However, in such multi-turn scenarios, particularly agentic search, the accumulation of retrieved documents and interaction history rapidly consumes the limited context window. This necessitates effective context compression strategies that preserve critical information while maintaining token efficiency, ensuring the agent operates within computational constraints without sacrificing performance.

\begin{figure*}[t]
    \centering
    \includegraphics[width=0.95\textwidth]{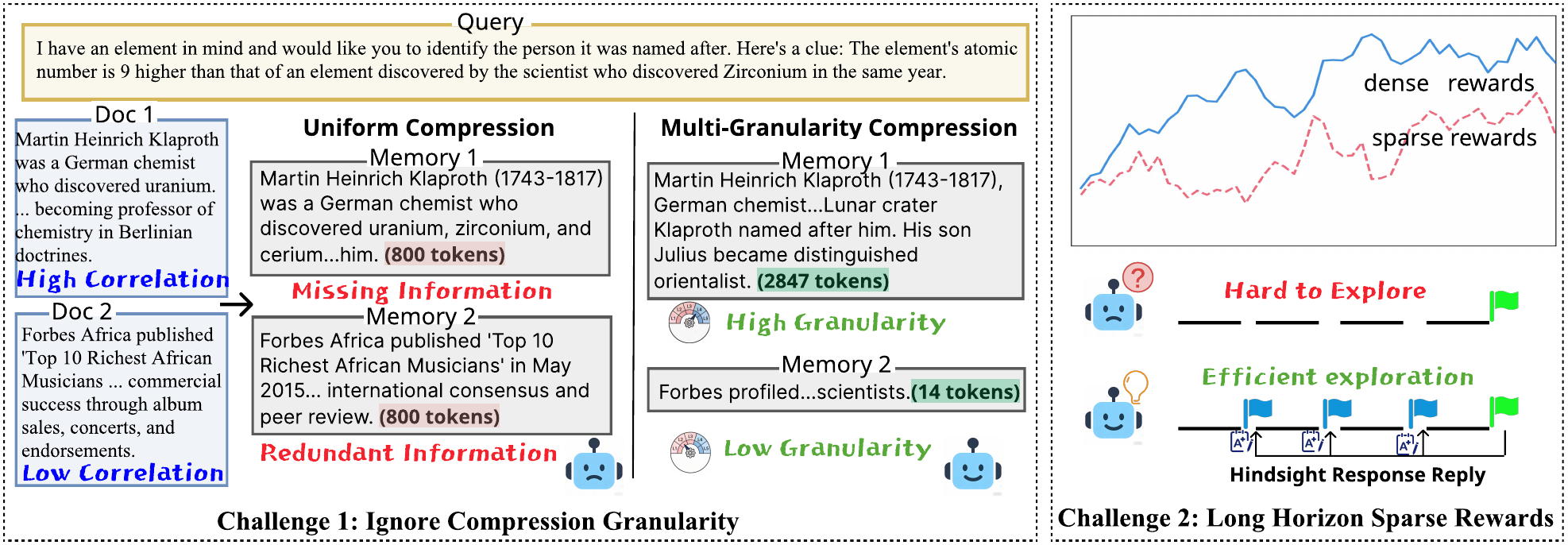}
    \vspace{-0.2cm}
    \caption{Challenges faced by current long-context compression and search agents, with ZipRL addressing these by incorporating the Multi-Granularity Mechanism and Hindsight Response Replay for improved compression performance.}
    \label{fig:ziprl_motivation}
    \vspace{-0.3cm}
\end{figure*}

Despite the critical role of context compression, existing approaches 
encounter two key challenges, as illustrated in 
Figure~\ref{fig:ziprl_motivation}. First, current methods typically 
employ uniform processing granularity~\citep{song2025r1, 
he2025webseer}, treating all retrieved segments with equal importance. 
This lack of relevance awareness is suboptimal from an 
information-theoretic perspective: as we demonstrate, uniform 
compression often fails to maximize the task-relevant information 
utility retained from the context for the target query. While 
irrelevant documents introduce noise, critical details in relevant 
documents may be lost if over-compressed. Second, optimizing 
compression policies via Reinforcement Learning (RL) is notoriously 
difficult due to sparse reward signals brought by long-horizon 
multi-turn interaction~\citep{yu2025memagent}. In multi-turn tasks, 
supervision is often only available at the final outcome (e.g., 
success or failure), making it challenging to credit specific 
intermediate compression actions. Although Process Reward Models 
(PRMs)~\citep{ma2023let, zhang2025lessons} attempt to mitigate this, 
their performance relies heavily on task-specific reward designs and 
may encourage greedy strategies.

To address these challenges, we introduce ZipRL, an adaptive multi-turn context compression framework designed to balance information preservation and efficiency. Grounded in the principle that multi-granularity compression preserves more task-relevant information utility, ZipRL enables the agent to actively perceive the relevance of retrieved content and dynamically select macro-level compression strategies (coarse-to-fine) via in-context prompts. Furthermore, to tackle the sparse reward problem in RL training with context compression, we propose Hindsight Response Replay (HRR) inspired by the idea of Hindsight Experience Replay (HER) in conventional RL. Instead of relying on expensive external PRMs, HRR utilizes a context-oriented heuristic metric to score compression segments efficiently. These scores are then integrated into Group Relative Policy Optimization (GRPO) through advantage re-shaping, densifying the training signals and allowing the model to learn optimal compression policies from the final outcome reward. Experimental results on five Multi-hop QA and Web Browsing benchmarks across multiple models of varying versions and parameter scales validate the effectiveness of our approach. It outperforms the strongest baselines by 27.9\% and 34.7\% on average for 4B and 8B parameter models, respectively, while maintaining superior performance in long-context interactions (up to 256 turns) with high token efficiency. Our contributions are summarized as follows:

\begin{itemize}[itemsep=0pt, parsep=0pt, topsep=2pt, partopsep=0pt]
    \item We identify the limitations of uniform processing in long-context tasks and theoretically validate that multi-granularity compression better preserves task-relevant information utility under general utility functions satisfying concavity and monotone marginal utility. Based on this, we propose an active compression mechanism that adaptively determines compression levels based on document-query relevance.
    \item We propose ZipRL, a novel framework that integrates adaptive 
    compression with robust RL optimization. We introduce Hindsight 
    Response Replay (HRR), which leverages a heuristic-based metric to 
    densify rewards via advantage re-shaping. We define a context-oriented 
    metric to score compression segments, leveraging reward densification 
    to optimize the average score towards a desired goal. HRR effectively 
    mitigates the sparse reward issue in multi-turn context compression 
    without the overhead of external reward models.
    \item Extensive experiments across five benchmarks in Web Browsing and Multi-hop QA tasks on multiple models of varying versions and parameter scales validate the effectiveness of our approach. Specifically, ZipRL outperforms the strongest same-scale specialized agents (e.g., ASearcher and AgentFold) by 27.9\% and 34.7\% average EM for Qwen3-4B and Qwen3-8B models respectively, during extreme extrapolation stress tests up to 256 turns ($12.8\times$ training horizon).
\end{itemize}

\section{Related Work}

\subsection{Context Compression for Agents}
Recent autonomous agents tackle long-horizon tasks via external retrieval~\citep{jin2025search, nguyen2025sfr}, necessitating effective context management to fit limited context windows. Existing methods employ external storage~\citep{zhong2024memorybank, salama2025meminsight}, internal memory updates~\citep{zhou2025mem1, ge2023context, huang2024recurrent}, or multi-agent collaboration~\citep{zhang2024chain}. Static prompt compression (e.g., LLMLingua~\citep{jiang2023llmlingua}, RECOMP~\citep{xu2023recomp}) reduces context pre-inference but lacks dynamic adaptability. While recent variable-rate compression methods (ACC-RAG~\citep{guo2025enhancing}, AttnComp~\citep{luo2025attncomp}, SARA~\citep{jin2025sara}) improve single-turn token efficiency, they operate in \emph{static} settings without multi-turn action adaptability or end-task RL feedback (see Appendix~\ref{appendix:rag_compression_comparison} for taxonomy). Crucially, most methods~\citep{yu2025memagent, xu2026mem} overlook the non-uniform relevance across retrieved documents. To address this, our multi-granularity mechanism dynamically adapts compression levels based on information density, thereby retaining more task-critical information.

\subsection{Sparse Rewards in RL}

Sparse rewards in long-horizon interactions severely hinder RL 
efficiency~\citep{riedmiller2018learning}. Traditional RL mitigates 
this via Hierarchical RL~\citep{pateria2021hierarchical, barto2003recent}, 
Curriculum Learning (CL)~\citep{wang2021survey}, or Hindsight Experience 
Replay (HER)~\citep{andrychowicz2017hindsight}. In the LLM context, 
approaches like CLPO~\citep{zhang2025clpo} and Process Reward Models 
(PRMs)~\citep{ma2023let, zhang2025lessons} attempt to densify training 
signals, but they often rely on expensive external models or task-specific 
designs. In contrast, ZipRL draws inspiration from HER to integrate 
hindsight-based reward densification directly into multi-turn agent 
training, alleviating the sparse reward bottleneck without external 
PRM overhead.
\section{Method}

\begin{figure*}[t]
    \centering
    \includegraphics[width=0.95\textwidth]{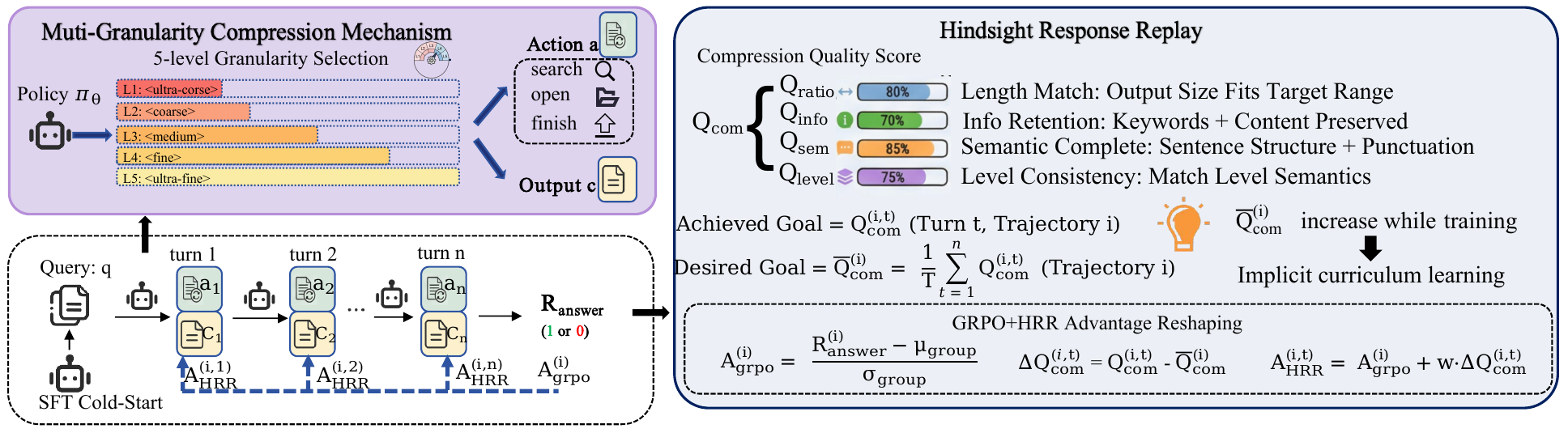}
    \vspace{-0.2cm}
    \caption{Overview of the ZipRL framework, including the Multi-Granularity Mechanism, Compression Quality Evaluation, and Hindsight Response Replay.}
    \label{fig:ziprl_framework}
    \vspace{-0.3cm}
\end{figure*}

Background on Hindsight Experience Replay (HER)~\citep{andrychowicz2017hindsight} and its relationship to our proposed HRR is provided in Appendix~\ref{appendix:her_preliminary}. As illustrated in Figure~\ref{fig:ziprl_framework}, we present the proposed adaptive long-text search agent along with the ZipRL algorithm. The Cold Start phase is detailed in Appendix~\ref{appendix:cold_start}, and the multi-granularity mechanism is described in Section~\ref{sec:multi_gran_mechanism}. Section~\ref{sec:compress_quality_eval} discusses the evaluation of compression quality, while Section~\ref{sec:hindsignt_response_reply} introduces the hindsight response replay mechanism. Finally, the reinforcement learning training objective is formalized in Section~\ref{sec:rl_training}.

\subsection{Cold Start}
\label{sec:cold_start}

Before RL training, we initialize the policy via a Cold Start SFT phase on 1,155 high-quality trajectories synthesized by GPT-4o, equipping the model with basic compression capabilities before RL optimization. Full details are in Appendix~\ref{appendix:cold_start}.

\subsection{Multi-Granularity Mechanism}
\label{sec:multi_gran_mechanism}

In real-world applications, documents and queries often exhibit varying 
degrees of relevance, which necessitates different compression 
granularities. To enable such adaptive compression behavior, we introduce 
a \textbf{Multi-Granularity Mechanism}. This mechanism leverages carefully 
designed prompts (see Appendix~\ref{appendix:prompt_details}) to guide the 
model in identifying the amount of query-relevant information within the 
retrieved text. It then selects an appropriate compression granularity, 
denoted by $g$, and allocates different weights according to textual 
relevance, ensuring that more valuable information is preserved. 
Specifically, we first analyze the text and select $k$ compression 
granularities (see Appendix~\ref{sec:granularity_selection} for details). 
For each granularity, we define a compression range 
$[\mathcal{L}^g_l, \mathcal{L}^g_h]$ based on the text length at that 
granularity. Documents with stronger relevance typically contain more 
important information; therefore, we apply a higher compression level and 
allocate a larger compression range to retain as much important information 
as possible. Conversely, we limit retention for weakly relevant documents 
to avoid noise. This adaptive strategy maximizes salient information 
retention while effectively filtering out irrelevancies.

\subsection{Compression Quality Evaluation}
\label{sec:compress_quality_eval}

To evaluate the quality of compressed context without relying on expensive human annotation or unstable external reward models, we design a multi-dimensional heuristic scoring function. A high-quality compressed summary should satisfy two complementary objectives: \textit{structural conformity} (adhering to length and format constraints) and \textit{semantic preservation} (retaining critical information). We construct the final score $Q_{\text{com}}$ based on four distinct dimensions. In this section, we outline the design rationale for each dimension; for detailed formulations, please refer to Appendix~\ref{appendix:score_details}.

\vspace{0.5em}
\noindent \textbf{Structure Objective:}
This objective enforces the length and style constraints specified by the target compression level $g$.

\paragraph{Compression Ratio Score $Q_{\text{ratio}}$.}
It measures whether the compressed length $l_y$ falls within the target range $[\mathcal{L}_l^g, \mathcal{L}_h^g]$ defined for level $g$. Unlike a simple binary check, we implement a soft penalty mechanism (see Figure~\ref{fig:r_ratio_plot} in Appendix~\ref{appendix:score_details}). Short outputs incur a linear penalty for information loss, while long outputs receive a non-linear penalty for redundancy.

\paragraph{Level Strategy Consistency Score $Q_{\text{level}}$.}
Different compression levels require different granularities, not just different lengths. For instance, Level 1 (Ultra-coarse) should consist of few concise sentences, whereas Level 5 (Ultra-fine) allows for detailed narration. We define expected bounds for sentence count $S_g$ and word count $W_g$ for each level. The score $R_{\text{level}}$ penalizes outputs that meet the character limit but violate the structural density.

\vspace{0.5em}
\noindent \textbf{Semantic Objective:}
Structural correctness is meaningless if the core meaning is lost or the text becomes incoherent. These dimensions evaluate the content quality.

\paragraph{Information Retention Score $Q_{\text{info}}$.}
To ensure the compressed text $y$ retains key information from the 
original text $x$ relevant to the query $q$, we employ a keyword-based 
overlap metric. We extract keyword sets $K_q$ from the query and $K_x$ 
from the original text via tokenization, stopword removal, and filtering 
for tokens with length $\ge 4$ characters. $Q_{\text{info}}$ is computed as a weighted combination of query-specific keyword coverage ($S_{\text{key}}$) and general content retention ($S_{\text{gen}}$), with $\lambda_1 = 0.7$ and $\lambda_2 = 0.3$. The denominator of $S_{\text{key}}$ uses $|K_q \cap x|$ rather than $|K_q|$, restricting evaluation to keywords \emph{actually present} in the source document to avoid penalizing irretrievable keywords. When $|K_q \cap x| = 0$, the system falls back to $S_{\text{gen}}$ alone. Full formulations are provided in Appendix~\ref{appendix:score_details}.

\paragraph{Semantic Completeness Score $Q_{\text{sem}}$.}
Aggressive compression often leads to broken sentences or linguistic fragmentation. We evaluate semantic completeness by checking for valid sentence terminators (punctuation), reasonable sentence structures, and vocabulary coherence. This acts as a linguistic validity filter, ensuring the compressed text remains natural and readable.

\vspace{0.5em}
\noindent \textbf{Total Score:}
The final compression quality score is a weighted sum of these four complementary dimensions:
\begin{equation}
    Q_{\text{com}} = \alpha_1 Q_{\text{ratio}} + \alpha_2 Q_{\text{level}} + \alpha_3 Q_{\text{info}} + \alpha_4 Q_{\text{sem}}
\end{equation}
In our implementation, we prioritize information retention and length compliance, setting weights as $\alpha=\{0.3, 0.1, 0.4, 0.2\}$ respectively. This composite score guides the model to balance the trade-off between brevity (structure) and utility (semantic).

\subsection{Hindsight Response Replay}
\label{sec:hindsignt_response_reply}

In multi-turn task settings, it is often difficult to distinguish the contribution of each individual turn solely based on the final reward signal. Inspired by the success of Hindsight Experience Replay (HER)~\citep{andrychowicz2017hindsight} in converting sparse binary feedback into dense learning signals, we propose \textbf{Hindsight Response Replay (HRR)} as an \emph{advantage reshaping} technique tailored for multi-turn LLM agent training. While HER achieves reward densification by relabeling goals and replaying transitions under substitute objectives, HRR achieves an analogous effect through a different mechanism: it uses a turn-level compression quality score to redistribute the trajectory-level advantage computed by GRPO, assigning denser credit to individual turns without requiring a replay buffer, goal-conditioned policy, or explicit goal substitution (see Table~\ref{tab:her_vs_hrr} for a detailed comparison). The key insight shared with HER is that non-optimal trajectories still contain valuable per-step information: even when the final task reward is zero, individual turns may exhibit high-quality compression behavior that should be reinforced. In GRPO, the advantage for each trajectory $ i $ is computed as:
\begin{equation}
A^{(i)}_{\text{GRPO}} = \frac{R^{(i)}_{\text{final}} - \mu_{\text{group}}}{\sigma_{\text{group}}},
\end{equation}
where $R^{(i)}_{\text{final}}$ denotes the final reward for trajectory $i$, and $\mu_{\text{group}}$ and $\sigma_{\text{group}}$ represent the mean and standard deviation of the group rewards, respectively. Each turn $ j $ in trajectory $ i $ uses the same advantage $ A^{(i)}_{\text{GRPO}} $, which makes it difficult for the agent to differentiate the contribution of each individual turn, potentially increasing the difficulty of training.

To address this limitation, HRR utilizes our compression quality score to reinterpret the per-turn contribution within each trajectory, alleviating the issue of sparse rewards by providing turn-level credit assignment. Concretely, we treat the achieved compression quality as a \emph{dynamic reference point} (analogous to how HER treats the achieved state as a substitute goal), allowing us to ``replay'' the advantage calculation by measuring each turn's distance to this reference. Specifically, for each trajectory $ i $ consisting of $ T $ turns, we compute the average compression quality score as the trajectory-level reference:
\begin{equation}
\bar{Q}_{\mathrm{com}}^{(i)} = \frac{1}{T} \sum_{j=1}^{T} Q_{\mathrm{com}}^{(i,j)}.
\end{equation}
For each turn in trajectory $ i $, we treat the corresponding compression quality score as the state and compare it with the expected target $ \bar{Q}_{\mathrm{com}}^{(i)} $ to evaluate its quality. This allows us to reallocate the trajectory-level advantage from GRPO at the turn level. Formally, for the $ j $-th turn in trajectory $ i $, the final advantage is defined as:
\begin{equation}
A^{(i,j)} = A_{\mathrm{GRPO}}^{(i)} + w \cdot \left( Q_{\mathrm{com}}^{(i,j)} - \bar{Q}_{\mathrm{com}}^{(i)} \right), \label{equ:advantage_adjust}
\end{equation}
where $ w $ is an advantage reshaping coefficient that controls the 
strength of hindsight relabeling. When the $ j $-th turn in trajectory 
$ i $ has a compression quality score greater than the trajectory 
average, the advantage value is increased. Conversely, if the score 
is smaller than the average, the advantage value is decreased. The 
magnitude of this adjustment is determined by the absolute difference 
in compression quality scores, thus providing a denser training signal 
for different turns in a trajectory and more accurately evaluating each 
turn's contribution. As training progresses, the model's compression 
quality improves (see Section~\ref{sec:experiments}), leading to a 
higher average compression quality score, thereby adaptively adjusting 
the expected training target. This approach alleviates the sparse reward 
problem and enables better credit assignment.

\subsection{Training Objectives}
\label{sec:rl_training}

Based on hindsight-relabeled dense advantages, ZipRL uses the following objective, which facilitates stable and efficient training through the $\mathcal{J}_{\text{ZipRL}}(\theta)$ objective:
{\small
\begin{equation}
\begin{aligned}
    \mathcal{J}_{\text{ZipRL}}(\theta) 
    &= \mathbb{E}_{q,\{o^i\}} \Bigg[
        \frac{1}{G}\sum_{i=1}^G \frac{1}{|o^i|}
        \sum_{j=1}^{T}\sum_{t=1}^{|o^{(i,j)}|} \\
    &\qquad\qquad\quad
        \left(m^{(i,j)}_t 
        - \beta\,\mathbb{D}_{\text{KL}}[\pi_\theta\|\pi_{\text{ref}}]\right)
    \Bigg], \\[4pt]
    m^{(i,j)}_t 
    &= \min\!\left(
        \rho^{(i,j)}_t A^{(i,j)},\;
        \operatorname{clip}\!\left(\rho^{(i,j)}_t,\, 1{-}\epsilon,\, 1{+}\epsilon\right)
        A^{(i,j)}
    \right), \\[4pt]
    \rho^{(i,j)}_t 
    &= \frac{
        \pi_\theta\!\left(o^{(i,j)}_t \mid q,\, o^{(i,j)}_{<t}\right)
    }{
        \pi_{\theta_{\text{old}}}\!\left(o^{(i,j)}_t \mid q,\, o^{(i,j)}_{<t}\right)
    }.
\end{aligned}
\end{equation}
}

\section{Theoretical Analysis}
We provide a utility-theoretic justification for multi-granularity compression, showing that relevance-aware resource allocation can improve expected downstream utility over uniform compression under the same average resource budget.

\begin{theorem}[Utility Advantage of Relevance-Aware Allocation]
\label{theorem:adaptive_compression}
Let $R\sim P$ be a relevance random variable supported on 
$\mathcal{R}\subseteq\mathbb{R}$. Let $\bar{\alpha}>0$ be the average 
resource budget, and let $\mathcal{A}\subseteq\mathbb{R}_{++}$ be an open 
interval with $\bar{\alpha}\in\mathcal{A}$. Consider a utility function 
$\Phi:\mathcal{R}\times\mathcal{A}\to\mathbb{R}$ such that 
$\Phi(\cdot,\alpha)$ is measurable for every $\alpha\in\mathcal{A}$, and 
$\Phi(r,\cdot)$ is differentiable and strictly concave on $\mathcal{A}$ for 
every $r\in\mathcal{R}$. Let the uniform allocation be $\alpha_{\rm uni}(r)\equiv\bar{\alpha}$, and let the relevance-aware adaptive allocation be $\alpha_{\rm ada}(r)=f(r)$, where $f:\mathcal{R}\to\mathcal{A}$ is measurable, non-decreasing, and satisfies
\[
\mathbb{E}[f(R)]=\bar{\alpha},
\qquad
\mathbb{P}\!\left(f(R)\neq\bar{\alpha}\right)>0 .
\]
Define the allocated marginal utility
$$
G_f(r)
=
\partial_{\alpha}\Phi(r,\alpha)\big|_{\alpha=f(r)} .
$$
Assume that $G_f$ is measurable and non-decreasing, and that
$\Phi(R,f(R))$, $\Phi(R,\bar{\alpha})$, $G_f(R)$, and $G_f(R)f(R)$ are 
integrable. Then
$$
\mathbb{E}\!\left[\Phi(R,f(R))\right]
>
\mathbb{E}\!\left[\Phi(R,\bar{\alpha})\right].
$$
\end{theorem}

\noindent\textbf{Remark.}
The monotonicity of $G_f$ means that the relevance-induced increase in marginal utility dominates the diminishing returns caused by larger allocations. Theorem~\ref{theorem:adaptive_compression} formalizes the
benefit of multi-granularity compression: under a fixed average resource budget, allocating more resources to more relevant documents yields a strict utility gain when the allocation is positively aligned with the marginal utility of additional compression resources. The proof is deferred to Appendix~\ref{app:proof_adaptive_compression}.

\section{Experiments}
\label{sec:experiments}

\subsection{Experimental Setup}

\begin{table*}[h!]
\centering
\footnotesize
\renewcommand{\arraystretch}{0.85}
\caption{Performance Comparison of Different Methods on Multi-hop QA and Web Browsing Tasks Across Five Datasets. BC-Plus denotes BrowseComp-plus. The \textbf{bold} and \underline{underline} indicate the best and second-best results, respectively.}
\label{tab:main_results}
\setlength{\tabcolsep}{6pt}
\resizebox{\textwidth}{!}{%
\begin{tabular}{@{}lcccccccccccc@{}}
\toprule
\multirow{2}{*}{\textbf{Model}} 
  & \multicolumn{2}{c}{\textbf{MusiQue}} 
  & \multicolumn{2}{c}{\textbf{SQuAD}} 
  & \multicolumn{2}{c}{\textbf{Frames}} 
  & \multicolumn{2}{c}{\textbf{Bamboogle}} 
  & \multicolumn{2}{c}{\textbf{BC-plus}} 
  & \multicolumn{2}{c}{\textbf{Average}} \\
\cmidrule(lr){2-3} \cmidrule(lr){4-5} \cmidrule(lr){6-7} 
\cmidrule(lr){8-9} \cmidrule(lr){10-11} \cmidrule(lr){12-13}
 & EM & F1 & EM & F1 & EM & F1 & EM & F1 & EM & F1 & EM & F1 \\
\midrule
\multicolumn{13}{l}{\textit{ReAct-based Methods}} \\
\midrule
Qwen3-235B-ReAct    & 18.2 & 30.4 & 17.8 & 32.3 & 14.5 & 27.4 & 40.1 & 52.5 & 14.0 & 19.5 & 20.9 & 32.4 \\
DeepSeek-V3.2-ReAct & 9.6  & 17.3 & 10.2 & 22.9 & 10.1 & 18.4 & 32.0 & 40.6 & 8.8  & 13.3 & 14.1 & 22.5 \\
Gemini-3-Pro-ReAct  
  & \textbf{32.2} & \textbf{44.7} 
  & \textbf{21.4} & \textbf{38.2} 
  & \textbf{24.8} & \textbf{38.7} 
  & \textbf{54.4} & \textbf{67.0} 
  & 22.8 & 27.8
  & \textbf{31.1} & \textbf{43.3} \\
GPT-4o-ReAct        
  & 11.8 & 28.9 
  & 8.6  & 24.5 
  & 15.8 & 32.6 
  & 39.2 & 55.3 
  & \textbf{25.2} & \textbf{31.6} 
  & 20.1 & 34.6 \\
LongCat-Flash-ReAct 
  & \underline{24.4} & \underline{37.5} 
  & \underline{18.6} & \underline{35.3} 
  & \underline{23.5} & \underline{38.0} 
  & \underline{48.8} & \underline{62.0} 
  & \underline{24.4} & \underline{28.5}
  & \underline{27.9} & \underline{40.3} \\
\midrule
\multicolumn{13}{l}{\textit{Summary-based Methods}} \\
\midrule
Qwen3-235B-Summary    
  & \underline{17.8} & \textbf{33.2}
  & \textbf{13.0} & 28.1 
  & 15.2 & 29.9 
  & \underline{40.0} & \underline{54.7} 
  & 14.0 & 19.2 
  & 20.0 & 33.0 \\
DeepSeek-V3.2-Summary 
  & 10.2 & 25.2 
  & 12.2 & \textbf{30.6}
  & 10.1 & 20.6 
  & 38.4 & 54.5 
  & 7.6  & 11.5 
  & 15.7 & 28.5 \\
Gemini-3-Pro-Summary  
  & 14.6 & 27.4 
  & 10.4 & 25.5 
  & \textbf{24.8} & \textbf{38.7} 
  & 35.2 & 49.4 
  & \underline{23.2} & 27.4
  & \underline{21.6} & \underline{33.7} \\
GPT-4o-Summary        
  & 3.6  & 18.7 
  & 5.6  & 19.7 
  & 10.3 & 24.1 
  & 20.0 & 37.2 
  & 17.6 & \underline{27.7}
  & 11.4 & 25.5 \\
LongCat-Flash-Summary 
  & \textbf{20.0} & \textbf{33.2}
  & \underline{12.8} & \underline{29.5}
  & \underline{18.5} & \underline{32.3} 
  & \textbf{52.8} & \textbf{66.1} 
  & \textbf{29.6} & \textbf{35.7}
  & \textbf{26.7} & \textbf{39.4} \\
\midrule
\multicolumn{13}{l}{\textit{Open-Source Large Model}} \\
\midrule
WebSailor-32B  & 22.0 & 34.6 & 17.4 & 37.3 & 14.0 & 29.5 & 30.4 & 45.1 & 17.6 & 25.6 & 20.3 & 34.4 \\
\midrule
\multicolumn{13}{l}{\textit{Specialized Agent (Qwen2.5-3B)}} \\
\midrule
Search-R1-3B
  & \underline{18.2} & \underline{22.4}
  & \underline{8.8}  & 17.6
  & 3.2              & 8.0
  & \underline{12.0} & 21.0
  & 0.8              & 2.8
  & \underline{8.6}  & 14.4 \\
WebSailor-3B
  & 5.0              & 12.8
  & 5.6              & \underline{17.8}
  & \underline{3.5}  & \underline{9.6}
  & \underline{12.0} & \underline{29.6}
  & \underline{2.8}  & \underline{3.8}
  & 5.8              & \underline{14.7} \\
\rowcolor{yellow!20}
ZipRL-3B (ours)
  & \textbf{27.8} & \textbf{37.5}
  & \textbf{25.2} & \textbf{40.3}
  & \textbf{13.4} & \textbf{23.5}
  & \textbf{44.0} & \textbf{56.9}
  & \textbf{22.5} & \textbf{24.8}
  & \textbf{26.6} & \textbf{36.6} \\
\midrule
\multicolumn{13}{l}{\textit{Specialized Agent (Qwen3-4B)}} \\
\midrule
NestBrowse-4B
  & 12.0             & 21.5
  & 14.6             & 28.5
  & 4.3              & 13.1
  & 15.2             & 24.9
  & 11.2             & 15.7
  & 11.5             & 20.7 \\
AgentFold-4B
  & \underline{16.0} & \underline{26.5}
  & \underline{19.4} & \underline{33.8}
  & \underline{14.3} & \underline{24.7}
  & \underline{44.0} & \underline{57.9}
  & \underline{15.6} & \underline{19.8}
  & \underline{21.9} & \underline{32.5} \\
\rowcolor{yellow!20}
ZipRL-4B (ours)
  & \textbf{31.4} & \textbf{40.4}
  & \textbf{23.0} & \textbf{39.0}
  & \textbf{14.6} & \textbf{26.2}
  & \textbf{47.2} & \textbf{60.9}
  & \textbf{24.0} & \textbf{26.9}
  & \textbf{28.0} & \textbf{38.7} \\
\midrule
\multicolumn{13}{l}{\textit{Specialized Agent (Qwen2.5-7B)}} \\
\midrule
Search-R1-7B
  & 14.6             & 19.0
  & 5.4              & 13.1
  & 5.1              & 9.5
  & 20.8             & 28.7
  & 0.4              & 3.4
  & 9.3              & 14.7 \\
WebSailor-7B
  & 7.6              & 14.9
  & 13.0             & 26.4
  & 7.2              & 14.1
  & 36.8             & 49.1
  & \underline{6.4}  & \underline{8.3}
  & 14.2             & 22.6 \\
ASearcher-7B
  & \textbf{32.6} & \textbf{42.9}
  & \underline{23.8} & \underline{39.5}
  & \underline{11.0} & \underline{21.6}
  & \underline{39.2} & \underline{51.5}
  & 6.0              & 8.2
  & \underline{22.5} & \underline{32.7} \\
\rowcolor{yellow!20}
ZipRL-7B (ours)
  & \underline{31.2} & \underline{41.4}
  & \textbf{25.6} & \textbf{43.0}
  & \textbf{15.8} & \textbf{26.0}
  & \textbf{48.0} & \textbf{61.6}
  & \textbf{19.5} & \textbf{21.5}
  & \textbf{28.0} & \textbf{38.7} \\
\midrule
\multicolumn{13}{l}{\textit{Specialized Agent (Qwen3-8B)}} \\
\midrule
NestBrowse-8B
  & 21.4             & 29.5
  & 19.2             & \underline{36.9}
  & 9.1              & 17.9
  & 33.6             & 47.4
  & 8.8              & 14.5
  & 18.4             & 29.2 \\
AgentFold-8B
  & \underline{21.6} & \underline{30.1}
  & \underline{20.8} & 34.9
  & \underline{12.3} & \underline{21.2}
  & \underline{42.4} & \underline{54.1}
  & \underline{15.6} & \underline{19.3}
  & \underline{22.5} & \underline{31.9} \\
\rowcolor{yellow!20}
ZipRL-8B (ours)
  & \textbf{36.8} & \textbf{46.7}
  & \textbf{25.8} & \textbf{42.4}
  & \textbf{18.1} & \textbf{29.3}
  & \textbf{49.8} & \textbf{64.7}
  & \textbf{20.8} & \textbf{23.0}
  & \textbf{30.3} & \textbf{41.2} \\
\midrule
\multicolumn{13}{l}{\textit{Specialized Agent (Qwen2.5-14B)}} \\
\midrule
Search-R1-14B
  & 11.8             & 16.6
  & 11.4             & 20.5
  & 7.2              & 13.5
  & 26.4             & 36.6
  & 1.6              & 4.8
  & 11.7             & 18.4 \\
ASearcher-14B
  & \textbf{36.6} & \textbf{47.6}
  & \textbf{24.4} & \textbf{41.6}
  & \underline{16.1} & \underline{27.7}
  & \underline{48.8} & \underline{60.0}
  & \underline{13.2} & \underline{17.9}
  & \underline{27.8} & \underline{39.0} \\
\rowcolor{yellow!20}
ZipRL-14B (ours)
  & \underline{30.8} & \underline{40.7}
  & \underline{23.2} & \underline{41.0}
  & \textbf{19.2} & \textbf{31.3}
  & \textbf{51.2} & \textbf{64.4}
  & \textbf{23.6} & \textbf{26.1}
  & \textbf{29.6} & \textbf{40.7} \\
\midrule
\multicolumn{13}{l}{\textit{Specialized Agent (Qwen3-14B)}} \\
\midrule
NestBrowse-14B
  & \underline{24.4} & \underline{35.2}
  & 15.2             & 30.3
  & \underline{13.8} & \underline{25.7}
  & 42.4             & \underline{56.9}
  & 7.6              & 11.7
  & 20.7             & 32.0 \\
AgentFold-14B
  & 22.2             & 32.9
  & \underline{21.8} & \underline{36.5}
  & 13.4             & 22.4
  & \underline{45.6} & 55.2
  & \underline{14.8} & \underline{16.9}
  & \underline{23.6} & \underline{32.8} \\
\rowcolor{yellow!20}
ZipRL-14B (ours)
  & \textbf{33.6} & \textbf{43.4}
  & \textbf{24.0} & \textbf{42.2}
  & \textbf{20.8} & \textbf{32.4}
  & \textbf{52.8} & \textbf{67.6}
  & \textbf{24.4} & \textbf{27.0}
  & \textbf{31.1} & \textbf{42.5} \\
\bottomrule
\end{tabular}
}
\end{table*}

\paragraph{Datasets.} We evaluate the agent on Web Browsing and Multi-hop QA. For Web Browsing, we use BrowseComp~\citep{wei2025browsecomp}, a dynamic web retrieval benchmark for long-context management. For Multi-hop QA, we use MusiQue~\citep{trivedi2022musique}, SQuAD~\citep{rajpurkar2016squad}, Frames~\citep{krishna2025fact}, and Bamboogle~\citep{press2023measuring}, which require aggregating evidence across retrieval steps. MusiQue, Frames, and Bamboogle are multi-hop datasets, whereas SQuAD is primarily single-hop. Following~\citet{jin2025search}, we evaluate SQuAD in an \emph{open-domain retrieval} setting without the source paragraph. ZipRL yields larger gains on harder multi-hop datasets, suggesting that adaptive compression is more beneficial for complex reasoning.

\paragraph{Baselines.} We compare our method against the following baselines:
\textbf{(1) ReAct}~\citep{yao2022react}, which uses large-scale closed-source models with context concatenation, setting the performance upper bound;
\textbf{(2) Summary-Only}, a training-free baseline using heuristic rules for context summarization;
\textbf{(3) Search Agent Methods} (Search-R1~\citep{jin2025search}, WebSailor~\citep{li2025websailor}, NestBrowse~\citep{li2025nested}, AgentFold~\citep{ye2025agentfold}, ASearcher~\citep{gao2025beyond}), which utilize dynamic context condensing or specialized agents for long-context reasoning.

\paragraph{Retrieval System.}
To ensure fair comparison, all models within each benchmark share an identical retrieval backend. For Multi-hop QA datasets (MusiQue, SQuAD, Frames, Bamboogle), we deploy E5-base-v2 (mean pooling, max query length 256) over the Wikipedia 2018 corpus with FAISS Flat Index (FP16). For BC-Plus, we use Qwen3-Embedding-8B (last-token pooling, max query length 8192) over the dedicated BC+ corpus with pre-computed local vector indices. Top-$k$, corpus, and index configuration are strictly identical for all compared methods on each dataset.

\paragraph{Experimental Details.} We use Qwen2.5-\{3B, 7B, 14B\} and Qwen3-\{4B, 8B, 14B\}~\citep{yang2025qwen3} as base models. All models are trained on an 8$\times$A100 (80GB) GPU cluster with at most 20 interaction turns, reward reshaping coefficient 0.2, learning rate $1\mathrm{e}{-6}$, and batch size 64. During evaluation, we use temperature 0.8 and top-$p$ 1.0, and report EM and F1~\citep{li2025search}. We adopt official baseline checkpoints, where ASearcher does not provide a 4B model and WebSailor provides 32B instead of 14B. Tool configurations, compression details, and additional experimental settings are provided in Appendix~\ref{appendix:tool_config} and Appendix~\ref{appendix:exp_details}.

\subsection{Main Results}
\paragraph{Performance Results.}
Table~\ref{tab:main_results} reports results on Multi-hop QA and Web Browsing benchmarks. Among 8B-scale agents, ZipRL-8B achieves 30.3\% average EM and 41.2\% F1, outperforming the much larger Qwen3-235B-ReAct baseline with $29\times$ fewer parameters and matching Gemini-3-Pro. Compared with similarly sized specialized agents, i.e., ASearcher-7B and AgentFold-8B, ZipRL-8B improves average EM by 7.8 points. This advantage is consistent across scales: ZipRL-4B and ZipRL-14B lead their parameter-matched groups, while ZipRL-32B further reaches 35.2\% average EM, improving over ZipRL-14B by 4.1 points (Appendix~\ref{appendix:32b_scaling}). ZipRL also substantially outperforms Summary-Only baselines across all datasets, indicating that its learned active compression policy preserves complex information dependencies beyond static heuristics. Moreover, on the same Qwen2.5 backbone series used by prior baselines, ZipRL-7B still yields a 5.5-point average EM gain over ASearcher-7B, further confirming that the improvements stem from algorithmic design rather than base-model differences (Appendix~\ref{appendix:qwen25_comparison}).

\begin{figure}[t]
    \centering
    \includegraphics[width=0.9\columnwidth]{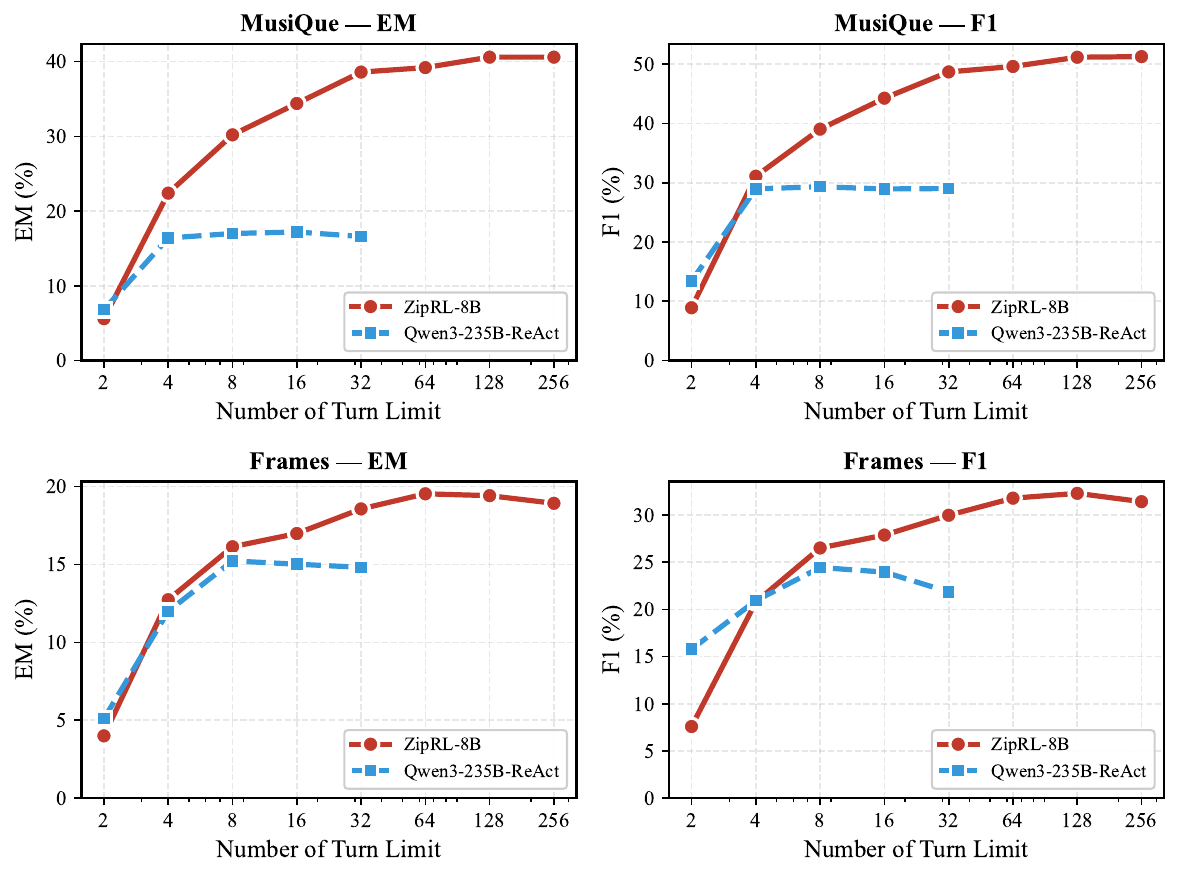}
    \caption{Performance variations of the MusiQue and Frames datasets under different turn constraints.}
    \label{fig:scalability}
\end{figure}

\begin{table*}[htbp]
\centering
\small
\caption{Ablation Study of ZipRL on the Qwen3-8B Model Across Five 
Datasets. BC-Plus denotes BrowseComp-plus. The \textbf{bold} and 
\underline{underline} indicate the best and second-best results, 
respectively.}
\label{tab:ablation_grouped}
\begin{tabular}{@{}lcccccccccccc@{}}
\toprule
\multirow{2}{*}{\textbf{Model Variant}} & \multicolumn{2}{c}{\textbf{MusiQue}} & \multicolumn{2}{c}{\textbf{SQuAD}} & \multicolumn{2}{c}{\textbf{Frames}} & \multicolumn{2}{c}{\textbf{Bamboogle}} & \multicolumn{2}{c}{\textbf{BC-Plus}} & \multicolumn{2}{c}{\textbf{Average}} \\
\cmidrule(lr){2-3} \cmidrule(lr){4-5} \cmidrule(lr){6-7} \cmidrule(lr){8-9} \cmidrule(lr){10-11} \cmidrule(lr){12-13}
 & EM & F1 & EM & F1 & EM & F1 & EM & F1 & EM & F1 & EM & F1 \\
\midrule
ZipRL (ours) & 36.8 & 46.7 & \textbf{25.8} & \textbf{42.4} & \textbf{18.1} & \textbf{29.3} & \textbf{49.8} & \underline{64.7} & \underline{20.8} & 23.0 & \textbf{30.3} & \textbf{41.2} \\
w/o RL & 29.0 & 39.7 & 21.0 & 36.4 & 16.0 & 26.4 & 45.6 & 59.8 & 20.4 & \underline{26.3} & 26.4 & 37.7 \\
w/o Level 2\&4 & 35.4 & 45.8 & 24.4 & 40.3 & \underline{17.7} & 28.2 & \underline{48.4} & \textbf{64.9} & 18.4 & 22.1 & 28.9 & 40.3 \\
w/o $Q_{\text{ratio}}$ & \underline{37.0} & \underline{47.1} & 23.6 & 38.7 & 17.6 & 27.8 & 47.0 & 61.4 & \textbf{21.6} & \textbf{26.5} & \underline{29.4} & 40.3 \\
w/o $Q_{\text{info}}$ & 32.4 & 41.4 & 20.8 & 37.2 & 15.3 & 27.1 & 47.2 & 58.7 & 18.8 & 23.2 & 26.9 & 37.5 \\
w/o $Q_{\text{sem}}$ & 32.4 & 42.5 & 23.0 & 38.5 & 16.9 & 28.3 & 46.4 & 62.0 & 20.0 & 24.1 & 27.7 & 39.1 \\
w/o $Q_{\text{level}}$ & \textbf{38.0} & \textbf{48.3} & \underline{25.0} & \underline{42.0} & 16.9 & \underline{28.8} & 46.4 & 59.9 & \underline{20.8} & 23.9 & \underline{29.4} & \underline{40.6} \\
\bottomrule
\end{tabular}
\end{table*}

\paragraph{Long-Horizon Scalability.} To assess robustness against context overload, we conduct extrapolation stress tests with extended interaction limits, as shown in Figure~\ref{fig:scalability}. We increase the maximum number of turns from 2 to 256 on MusiQue and Frames, where 256 turns serves as a stress-test boundary rather than a typical QA length. ZipRL-8B consistently outperforms the larger Qwen3-235B-ReAct baseline. While Qwen3-235B-ReAct saturates after 16 turns, ZipRL continues improving up to 256 turns, despite being trained with at most 20 turns. This $12.8\times$ horizon extrapolation suggests that ZipRL learns relevance-aware retention rather than overfitting to short interaction patterns, enabling robust long-horizon deployment under accumulated context noise.

\begin{figure}[t]
    \centering
    \begin{minipage}[b]{0.48\columnwidth}
        \centering
        \includegraphics[width=\textwidth]{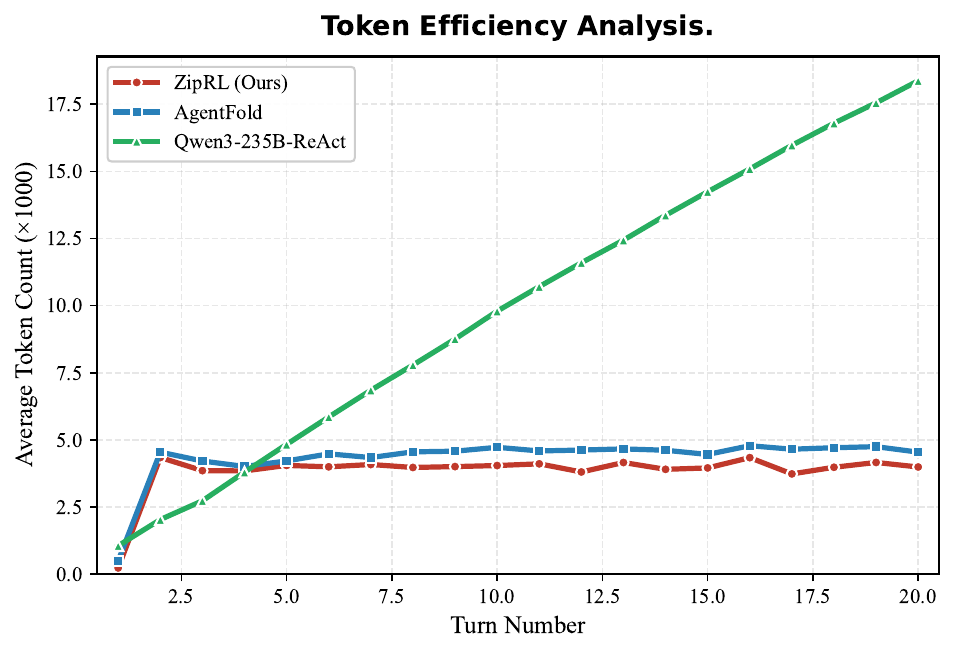}
    \end{minipage}
    \begin{minipage}[b]{0.44\columnwidth}
        \centering
        \includegraphics[width=\textwidth]{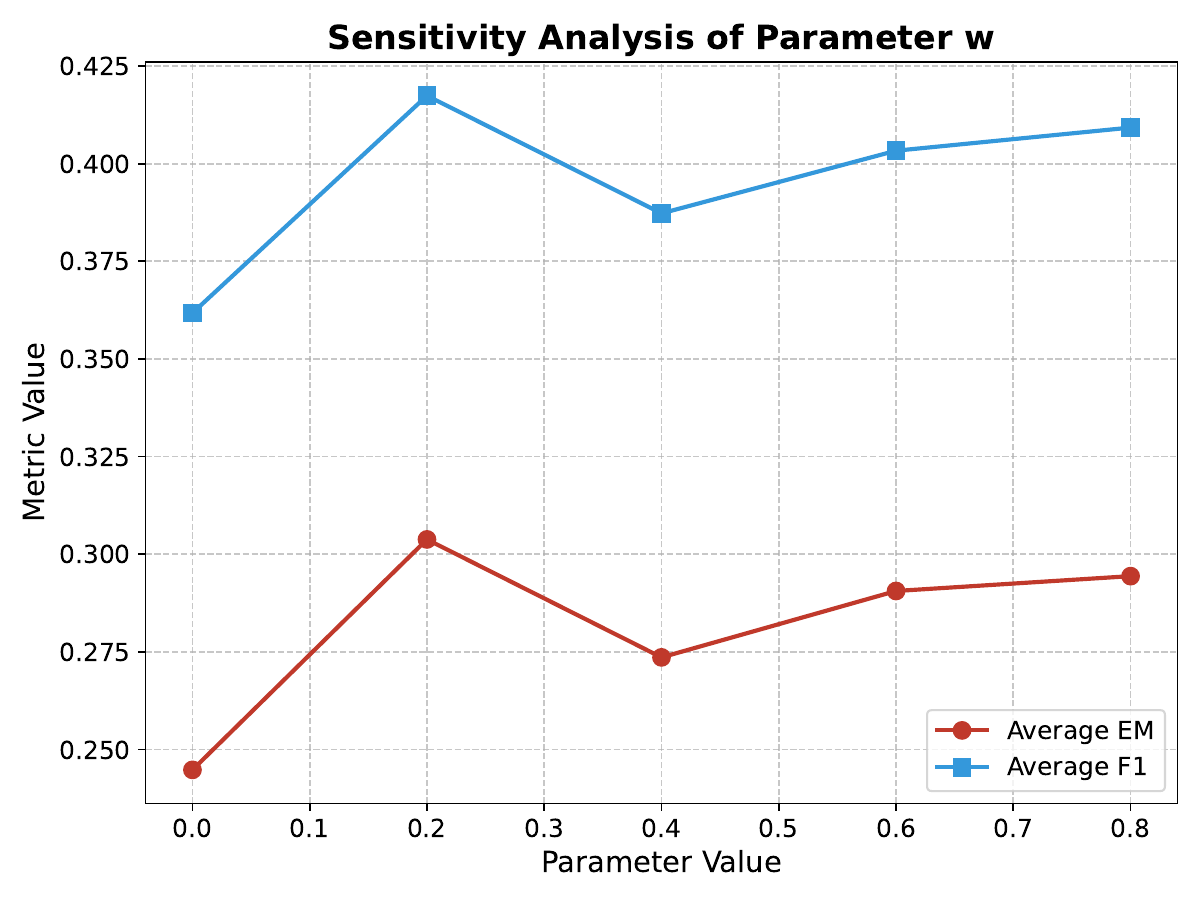}
    \end{minipage}
\caption{(Left) Average token count variation over training iterations. (Right) Sensitivity analysis of parameter $w$ across five datasets with ZipRL.}
    \label{fig:token_efficiency_and_w}
\end{figure}

\paragraph{Token Efficiency Analysis.} As shown in Figure~\ref{fig:token_efficiency_and_w} (left), ZipRL achieves strong token efficiency. Compared with ReAct-based large models, ZipRL-8B substantially reduces token usage while preserving performance. Its adaptive multi-granularity compression further outperforms AgentFold by better balancing information retention and computational cost. A detailed analysis of premature termination is provided in Appendix~\ref{appendix:trajectory_analysis}.

\begin{figure}[t]
    \centering
    \includegraphics[width=0.9\columnwidth]{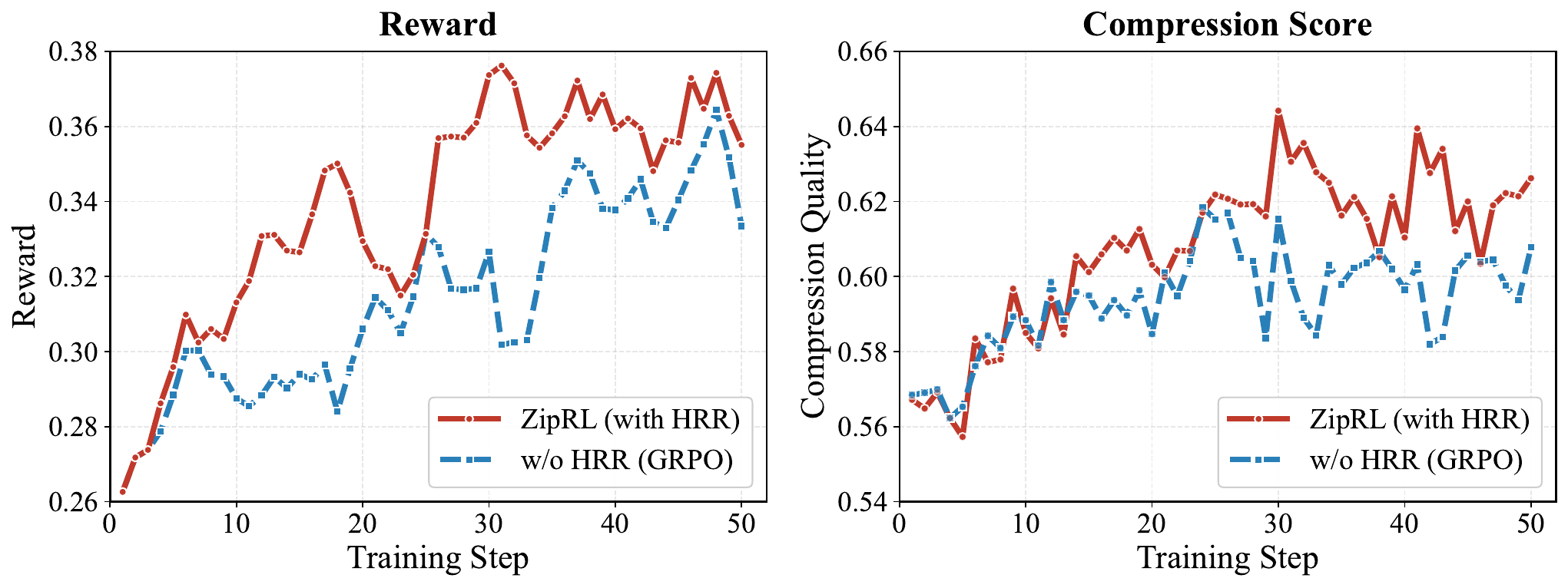}
    \caption{Training dynamics of ZipRL. (Left) Average reward and (Right) compression quality score.}
    \label{fig:reward_and_compression}
\end{figure}

\paragraph{Fine-Grained Behavior Statistics.} Fine-grained behavior statistics are provided in Appendix~\ref{appendix:agent_behavior_analysis} (Table~\ref{tab:behavior_stats}). ZipRL achieves a Finish Rate of 0.85 and adaptively increases retrieval intensity with question difficulty, while sustaining deep exploration on BC-Plus with 13.8 turns, avoiding under-retrieval and context-overload issues observed in baselines.

\section{Ablation Studies}
To quantify the contribution of each component, we conduct ablation studies with Qwen3-8B as the base model, as shown in Table~\ref{tab:ablation_grouped}. First, removing the RL phase and using only the cold-start model leads to a substantial performance drop, highlighting the importance of reinforcement learning through environment interaction. Second, reducing the compression granularity by removing Levels 2 and 4 consistently lowers average EM and F1, indicating the benefit of maintaining five compression levels. Finally, we ablate each dimension of the Compression Quality Score, i.e., $Q_{\text{ratio}}$, $Q_{\text{level}}$, $Q_{\text{info}}$, and $Q_{\text{sem}}$. Removing any dimension degrades performance, with $Q_{\text{info}}$ and $Q_{\text{sem}}$ causing the largest drops. Results demonstrate that ZipRL benefits from fine-grained compression control and quality-aware reward design, particularly through information retention and semantic completeness.

\section{Conclusion}

Despite recent progress in LLM-based long-text search, existing models still lack proactive awareness of document relevance and suffer from sparse rewards in long-horizon tasks. We introduce ZipRL to address these challenges through a multi-granularity mechanism that adaptively assigns compression levels according to query relevance. We further design a rule-based evaluation scheme and incorporate compression scores into Hindsight Response Replay to mitigate reward sparsity during multi-turn training. Theoretically, we analyze the benefits of both multi-granularity compression and HRR. Extensive experiments across multiple models of varying versions and parameter scales on five Web Browsing and Multi-hop QA datasets validate the effectiveness of our approach and show that ZipRL consistently outperforms strong baselines.

\section*{Limitations}

We note three limitations of ZipRL. (1) Language: $Q_{\text{info}}$ relies on English stopwords, potentially degrading in multilingual or specialized domains (e.g., legal, code). (2) Robustness: $Q_{\text{com}}$ ignores document credibility; EM drops heavily (85--99\%) under fully adversarial retrieval (Appendix~\ref{appendix:noise_robustness}). (3) Generalization: Relying on a single QA corpus for cold-start may limit adaptation to structured or code-heavy tasks.

\section*{Ethical Statements}

Our experiments rely strictly on publicly available datasets, involving no sensitive user data. Large language models were used solely for language editing during manuscript preparation. All research outcomes and intellectual content remain the exclusive work of the authors.



\bibliography{sample-base}

\appendix

\section{Cold Start Details}
\label{appendix:cold_start}

Before RL training, we implement a Cold Start phase to initialize the policy. We constructed a demonstration dataset, $\mathcal{D}_{\text{demo}}$, by sampling 3,000 queries from ASearcher-LRM-35k~\citep{gao2025beyond} and synthesizing trajectories via GPT-4o. To ensure high quality, we applied two filters: (1) \textbf{Correctness}: retaining only correct trajectories, and (2) \textbf{Format}: discarding outputs with structural violations. This yielded 1,155 valid trajectories, which were then decomposed into transition-level samples. The model is optimized using the standard Supervised Fine-Tuning (SFT) objective:
$$
\mathcal{L}_{\text{SFT}}(\theta) = -\mathbb{E}_{(q, o) \sim \mathcal{D}_{\text{demo}}} \left[ \sum_{t=1}^{|o|} \log \pi_{\theta}(o_t \mid q, o_{<t}) \right]
$$
where $\pi_{\theta}$ denotes the policy network, and $o = (o_1, o_2, \ldots, o_{|o|})$ is the output sequence. Each $o_t$ represents the observation or generated context token at step $t$ (including retrieved documents, reasoning traces, and tool outputs accumulated during the interaction). This phase equips the model with essential capabilities, adequately preparing it for subsequent RL optimization.

\section{Diagnosing the Performance Plateau at 256 Turns}
\label{appendix:plateau_diagnosis}

To understand what limits performance as interaction turns increase beyond $\sim$128, we conduct a targeted diagnostic analysis on MusiQue along three axes. \textbf{(1) Dataset ceiling:} An oracle experiment (providing all gold-standard evidence documents directly to the model without retrieval) achieves $51.3\%$ EM on MusiQue, substantially above ZipRL-8B's $36.8\%$ at 256 turns, so the dataset ceiling is not the primary bottleneck. \textbf{(2) Retriever saturation:} We track the \emph{novel-entity rate}: the fraction of retrieved documents at each turn that introduce at least one named entity not seen in any prior retrieved document. This rate drops from $0.81$ at turn 10 to $0.29$ at turn 128 and reaches $0.17$ at turn 256, indicating that the retriever returns largely redundant results in later turns. \textbf{(3) Compression capacity:} The context compression ratio (compressed tokens / raw retrieved tokens) remains stable at $\approx 0.41$--$0.44$ throughout all turn budgets, ruling out compression capacity as a limiting factor. Together, these results indicate that \textbf{retriever saturation} (i.e., the exhaustion of novel information in the retrieval corpus) is the dominant bottleneck at ultra-long horizons.

\section{Preliminary: Hindsight Experience Replay}
\label{appendix:her_preliminary}

Hindsight Experience Replay (HER)~\citep{andrychowicz2017hindsight} is a powerful technique designed to improve sample efficiency in sparse-reward RL tasks. In many real-world scenarios, agents receive meaningful feedback only upon achieving specific goals, making traditional RL inefficient due to the scarcity of reward signals. HER addresses this challenge by transforming unsuccessful episodes into useful training data through the concept of substitute goals. Formally, a trajectory under HER is represented as:
\begin{equation}
\tau = (s, a, r, s', g, g')
\end{equation}
where $ s $ denotes the current state, $ a $ the executed action, $ r $ the reward with respect to the original goal $ g $, $ s' $ the next state, and $ g' $ a substitute goal. The reward is recomputed with respect to $ g' $, denoted as $ r' $, allowing the agent to learn from outcomes that were originally considered failures. This relabeling converts sparse, binary rewards into denser signals, accelerating learning.

\paragraph{Relationship Between HER and Our Approach.}
While HER inspires our method philosophically (both convert sparse binary feedback into denser learning signals), the two differ substantially in mechanism (see Table~\ref{tab:her_vs_hrr} below). HER operates on goal-conditioned MDPs by relabeling the desired goal post-hoc and replaying stored transitions under the new goal, requiring an explicit replay buffer and a goal-conditioned reward function. In contrast, our proposed Hindsight Response Replay (HRR) operates within the GRPO framework by reshaping the trajectory-level advantage at turn granularity using a heuristic compression quality score. HRR requires neither a replay buffer nor goal substitution; instead, it treats the achieved compression quality as a reference point for credit assignment, redistributing the sparse final reward across turns proportional to each turn's compression quality relative to the trajectory mean. The shared philosophical principle is that \emph{information from non-optimal outcomes should not be discarded}: HER extracts value by asking ``what goal \emph{would} this trajectory have achieved?'', while HRR extracts value by asking ``which turns contributed more to the compression objective within this trajectory?''

\section{Detailed Comparison: HER vs.\ HRR}
\label{appendix:her_hrr_table}

Table~\ref{tab:her_vs_hrr} summarizes the key distinctions between Hindsight Experience Replay (HER) and our proposed Hindsight Response Replay (HRR).

\begin{table}[h]
\centering
\footnotesize
\setlength{\tabcolsep}{4pt}
\caption{Comparison between Hindsight Experience Replay (HER) and Hindsight Response Replay (HRR).}
\label{tab:her_vs_hrr}
\resizebox{\columnwidth}{!}{%
\begin{tabular}{@{}lp{2.5cm}p{3cm}@{}}
\toprule
\textbf{Dimension} & \textbf{HER} & \textbf{HRR (Ours)} \\
\midrule
Signal source        & Relabeled goal reward                  & Compression quality score $Q_{\text{com}}$ \\
Core mechanism       & Goal substitution + transition replay  & Advantage reshaping \\
Replay buffer        & Required                               & Not required \\
Goal conditioning    & Explicit (goal-conditioned policy)     & Implicit (trajectory-mean baseline) \\
Reward densification & Via counterfactual goal relabeling     & Via turn-level quality comparison \\
Action space         & Continuous control / discrete actions  & Token-level generation \\
Shared philosophy    & \multicolumn{2}{c}{Extract learning signal from non-optimal experiences} \\
\bottomrule
\end{tabular}%
}
\end{table}




\section{Proof of Theorem~\ref{theorem:adaptive_compression}}
\label{app:proof_adaptive_compression}

\begin{theorem}[Utility Advantage of Relevance-Aware Allocation]
Let $R\sim P$ be a relevance random variable supported on 
$\mathcal{R}\subseteq\mathbb{R}$. Let $\bar{\alpha}>0$ be the average 
resource budget, and let $\mathcal{A}\subseteq\mathbb{R}_{++}$ be an open 
interval with $\bar{\alpha}\in\mathcal{A}$. Consider a utility function 
$\Phi:\mathcal{R}\times\mathcal{A}\to\mathbb{R}$ such that 
$\Phi(\cdot,\alpha)$ is measurable for every $\alpha\in\mathcal{A}$, and 
$\Phi(r,\cdot)$ is differentiable and strictly concave on $\mathcal{A}$ for 
every $r\in\mathcal{R}$. Let the uniform allocation be $\alpha_{\rm uni}(r)\equiv\bar{\alpha}$, and let the relevance-aware adaptive allocation be $\alpha_{\rm ada}(r)=f(r)$, where $f:\mathcal{R}\to\mathcal{A}$ is measurable, non-decreasing, and satisfies
$$
\mathbb{E}[f(R)]=\bar{\alpha}, 
\qquad 
\mathbb{P}\!\left(f(R)\neq\bar{\alpha}\right)>0 .
$$
Define the allocated marginal utility
$$
G_f(r)
=
\partial_{\alpha}\Phi(r,\alpha)\big|_{\alpha=f(r)} .
$$
Assume that $G_f$ is measurable and non-decreasing, and that
$\Phi(R,f(R))$, $\Phi(R,\bar{\alpha})$, $G_f(R)$, and $G_f(R)f(R)$ are 
integrable. Then
$$
\mathbb{E}\!\left[\Phi(R,f(R))\right]
>
\mathbb{E}\!\left[\Phi(R,\bar{\alpha})\right].
$$
\end{theorem}

\begin{proof}
For each fixed relevance value $r$, define the utility difference between
adaptive and uniform allocation as
$$
\Delta(r)
=
\Phi(r,f(r))-\Phi(r,\bar{\alpha}).
$$
Since $\alpha\mapsto \Phi(r,\alpha)$ is differentiable and strictly concave on
$\mathcal{A}$, the first-order concavity inequality gives
$$
\Phi(r,\bar{\alpha})
\le
\Phi(r,f(r))
+
G_f(r)\big(\bar{\alpha}-f(r)\big),
$$
with strict inequality whenever $f(r)\neq \bar{\alpha}$. Therefore,
$$
\Delta(r)
\ge
G_f(r)\big(f(r)-\bar{\alpha}\big).
$$

To handle strictness rigorously, define the slack term
$$
S(r)
=
\Delta(r)
-
G_f(r)\big(f(r)-\bar{\alpha}\big).
$$
By strict concavity, $S(r)\ge 0$, and $S(r)>0$ whenever
$f(r)\neq\bar{\alpha}$. Since the adaptive allocation is non-trivial,
$$
\mathbb{P}\big(f(R)\neq \bar{\alpha}\big)>0.
$$
Thus, $S(R)$ is non-negative and strictly positive on an event with positive
probability. Under the stated integrability assumptions,
$$
\mathbb{E}[S(R)]>0.
$$
Taking expectation over $R$ gives
$$
\begin{aligned}
\mathbb{E}[\Delta(R)]
&=
\mathbb{E}\!\left[
G_f(R)\big(f(R)-\bar{\alpha}\big)
\right]
+
\mathbb{E}[S(R)]  \\
&>
\mathbb{E}\!\left[
G_f(R)\big(f(R)-\bar{\alpha}\big)
\right].
\end{aligned}
$$

It remains to show that the last expectation is non-negative. Let
$$
A=f(R),
\qquad
M=G_f(R).
$$
Using the budget constraint $\mathbb{E}[A]=\bar{\alpha}$, we have
$$
\begin{aligned}
\mathbb{E}\!\left[M(A-\bar{\alpha})\right]
&=
\mathbb{E}[MA]
-
\bar{\alpha}\mathbb{E}[M]  \\
&=
\mathbb{E}[MA]
-
\mathbb{E}[A]\mathbb{E}[M]  \\
&=
\operatorname{Cov}(M,A).
\end{aligned}
$$
Because both $f$ and $G_f$ are non-decreasing functions of $r$, this covariance
is non-negative. Let $R'$ be an independent copy of $R$. Then
$$
\begin{aligned}
2\operatorname{Cov}(M,A)
&=
\mathbb{E}\!\left[
\big(G_f(R)-G_f(R')\big)
\right. \\
&\qquad\left.
\cdot
\big(f(R)-f(R')\big)
\right].
\end{aligned}
$$
Since both $G_f$ and $f$ are non-decreasing, we have
$$
\big(G_f(R)-G_f(R')\big)
\big(f(R)-f(R')\big)
\ge 0
$$
almost surely. Therefore,
$$
\operatorname{Cov}(M,A)\ge 0.
$$
Consequently,
$$
\mathbb{E}\!\left[
G_f(R)\big(f(R)-\bar{\alpha}\big)
\right]
\ge 0.
$$
Combining the above inequalities yields
$$
\mathbb{E}[\Delta(R)]>0.
$$
Equivalently,
$$
\mathbb{E}\big[\Phi(R,f(R))\big]
>
\mathbb{E}\big[\Phi(R,\bar{\alpha})\big].
$$
This completes the proof.
\end{proof}

\section{Human Evaluation of Compression Quality}
\label{appendix:human_eval}

To validate that $Q_{\text{com}}$ correlates with human-perceived compression quality, we conducted a targeted human evaluation study. We randomly sampled 100 compression outputs from ZipRL-8B inference trajectories on the MusiQue and Frames test sets (50 per dataset, 20 per compression level). Three annotators with graduate-level NLP backgrounds rated each compressed output on two dimensions: \textbf{Fidelity} (1--5) and \textbf{Coherence} (1--5). Inter-annotator agreement was measured by Cohen's $\kappa$, yielding $\kappa = 0.73$ for Fidelity and $\kappa = 0.71$ for Coherence, indicating substantial agreement.

\begin{table}[h]
\centering
\caption{Spearman correlation between $Q_{\text{com}}$ sub-scores and human ratings ($n=100$ samples).}
\label{tab:human_eval}
\setlength{\tabcolsep}{8pt}
\begin{tabular}{@{}lcc@{}}
\toprule
\textbf{Metric} & \textbf{Fidelity $\rho$} & \textbf{Coherence $\rho$} \\
\midrule
$Q_{\text{info}}$             & 0.64          & 0.29 \\
$Q_{\text{sem}}$              & 0.27          & 0.73 \\
$Q_{\text{ratio}}$            & 0.43          & 0.46 \\
$Q_{\text{level}}$            & 0.39          & 0.54 \\
\midrule
$Q_{\text{com}}$ (composite) & \textbf{0.68} & \textbf{0.69} \\
\bottomrule
\end{tabular}
\end{table}

The weighted composite $Q_{\text{com}}$ achieves the highest correlation on both dimensions ($\rho = 0.68$ / $0.69$), outperforming all individual sub-scores and justifying the multi-dimensional design.

\section{Details of Compression Quality Scoring}
\label{appendix:score_details}

\subsection{Compression Ratio Score $Q_{\text{ratio}}$}

Let $l_y$ be the character length of the generated compressed text $y$, and $[\mathcal{L}_l^g, \mathcal{L}_h^g]$ be the target length interval for compression level $g$. We employ a piecewise function:
\begin{equation}
Q_{\text{ratio}} =
\begin{cases}
1.0, & \text{if } \mathcal{L}_l^g \leq l_y \leq \mathcal{L}_h^g \\
0.8 \cdot \frac{l_y}{\mathcal{L}_l^g}, & \text{if } l_y < \mathcal{L}_l^g \\
1.0 - 0.3 \cdot \frac{l_y - \mathcal{L}_h^g}{0.5 \mathcal{L}_h^g}, & \text{if } \mathcal{L}_h^g < l_y \leq 1.5 \mathcal{L}_h^g \\
0.3 \cdot \frac{1.5 \mathcal{L}_h^g}{l_y}, & \text{if } l_y > 1.5 \mathcal{L}_h^g
\end{cases}
\end{equation}

\begin{figure}[htbp]
    \centering
    \includegraphics[width=1\columnwidth]{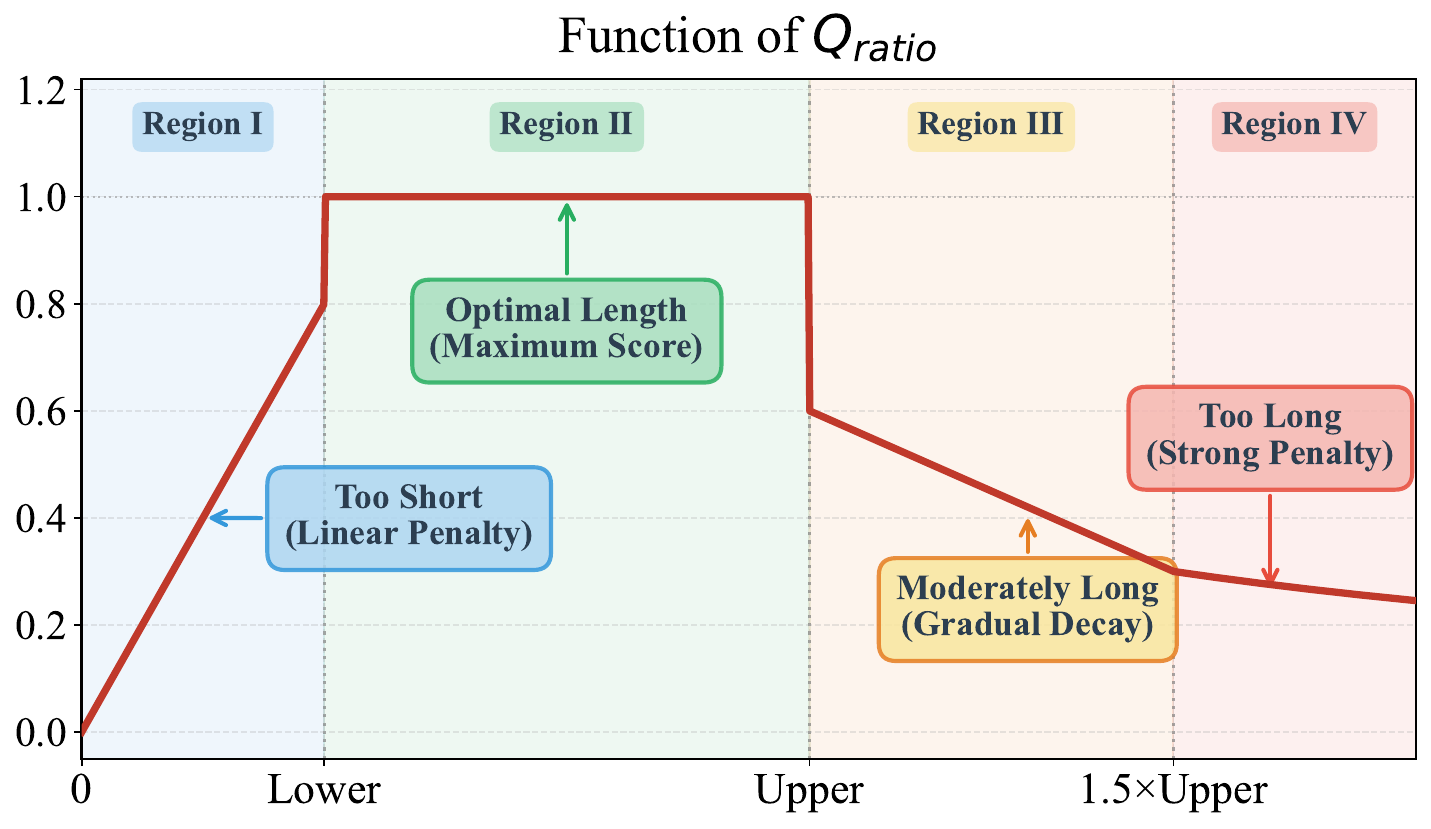}
    \caption{Compression ratio score versus text length.}
    \label{fig:r_ratio_plot}
\end{figure}

\subsection{Information Retention Score $Q_{\text{info}}$}

Let $K_q$ and $K_x$ denote the sets of keywords extracted from the query $q$ and the original text $x$, respectively. The retention score is:
\begin{equation}
    Q_{\text{info}} = 0.7 S_{\text{key}} + 0.3 S_{\text{gen}},
\end{equation}
where:
\begin{align}
    S_{\text{key}} &= \frac{|K_q \cap y|}{|K_q \cap x| + \epsilon}, \\
    S_{\text{gen}} &= 0.6 \cdot \frac{|y \cap K_x|}{|K_x| + \epsilon} + 0.4 \cdot \min\left(1.0, 2.0 \cdot \frac{|y|}{|x|}\right).
\end{align}

\subsection{Semantic Completeness Score $Q_{\text{sem}}$}

We evaluate linguistic validity using a multiplicative penalty model:
\begin{equation}
    Q_{\text{sem}} = 1.0 \cdot p_{\text{str}} \cdot p_{\text{len}} \cdot p_{\text{punc}},
\end{equation}
with the penalties defined as:
\begin{itemize}
    \item $p_{\text{str}} = 0.7$ if the text contains only a single sentence (unless level $g=1$), and $0.3$ if no valid sentence structure is detected; otherwise $1.0$.
    \item $p_{\text{len}} = 0.5$ if word count $< 20$, and $0.9$ if word count $> 1200$; otherwise $1.0$.
    \item $p_{\text{punc}} = 0.8$ if no sentence terminators are present; otherwise $1.0$.
\end{itemize}

\subsection{Level Strategy Consistency Score $Q_{\text{level}}$}

Given the target bounds $[\mathcal{W}_{min}^g, \mathcal{W}_{max}^g]$ and maximum sentence count $S_{max}^g$:
\begin{equation}
    Q_{\text{level}} = 0.7 S_{\text{word}} + 0.3 S_{\text{sent}},
\end{equation}
where:
\begin{equation}
    S_{\text{word}} =
    \begin{cases}
    1.0, & \text{if } \mathcal{W}_{min}^g \le W_y \le \mathcal{W}_{max}^g \\
    W_y / W_{min}^g, & \text{if } W_y < \mathcal{W}_{min}^g \\
    0.7, & \text{if } \mathcal{W}_{max}^g < W_y \le 1.5 \mathcal{W}_{max}^g \\
    0.3, & \text{otherwise}
    \end{cases}
\end{equation}
\begin{equation}
    S_{\text{sent}} = \min\left(1.0, \max\left(0.3, \frac{S_{max}^g}{S_y}\right)\right).
\end{equation}

\section{Analysis of Learned Compression Strategy}
\label{appendix:compression_level_distribution}

To understand the adaptive compression behavior learned by ZipRL, we analyzed the compression level selection patterns across 624 multi-turn interaction trajectories in our held-out test set.

\begin{table}[t]
\centering
\caption{Distribution of compression level selections by ZipRL-8B across 624 test trajectories.}
\label{tab:compression_level_dist}
\footnotesize
\setlength{\tabcolsep}{8pt}
\begin{tabular}{@{}lcc@{}}
\toprule
\textbf{Compression Level} & \textbf{Count} & \textbf{Percentage} \\
\midrule
Level 1 (Ultra-coarse) & 65 & 10.42\% \\
Level 2 (Coarse) & 10 & 1.60\% \\
Level 3 (Medium) & 339 & 54.33\% \\
Level 4 (Fine) & 51 & 8.17\% \\
Level 5 (Ultra-fine) & 159 & 25.48\% \\
\midrule
\textbf{Total} & \textbf{624} & \textbf{100.00\%} \\
\bottomrule
\end{tabular}
\end{table}

The predominance of Level 3 (54.33\%) shows the model learned that moderate compression suffices for most documents while conserving computational resources. The distribution aligns with Theorem~\ref{theorem:adaptive_compression}, validating that adaptive allocation outperforms uniform compression.

\section{Tool Configuration}
\label{appendix:tool_config}

To enable the agent to perform automated information retrieval tasks, we designed three tools.

\textbf{Search Tool}: Accepts textual queries and returns the most relevant document fragments, document IDs, and URLs.

\textbf{Open-Page Tool}: Accepts either a URL or a document ID and returns the complete webpage content, enabling detailed retrieval following search results.

\textbf{Finish Tool}: Invoked when the agent has obtained a definitive answer or can no longer proceed.

By coordinating these three tools, the agent executes the full ``Search, Browse, Analyze, and Output'' pipeline for automated information retrieval. ZipRL enforces a ``compress-on-every-turn'' mechanism: regardless of the specific tool invoked, the agent outputs a compressed reasoning state or observation summary before deciding on the next action.

\section{Experimental Details}
\label{appendix:exp_details}

\subsection{Datasets for Experiments}
\label{appendix:dataset_details}

Table~\ref{tab:datasets_statics} summarizes the statistics of all 
training and test datasets used in our experiments.

\begin{table}[t]
\centering
\small
\setlength{\tabcolsep}{4pt}
\renewcommand{\arraystretch}{1.1}
\begin{tabular}{llll}
\toprule
\textbf{Dataset} & \textbf{Description} & \textbf{Samples} & \textbf{Usage} \\
\midrule
ASearcher-LRM & Long-Horizon Search & 4,500 & Train \\
BC-Plus & Web Browsing & 830 & Test \\
MusiQue & Multi-hop QA & 2,459 & Test \\
SQuAD & Multi-hop QA & 2,694 & Test \\
Frames & Multi-hop QA & 824 & Test \\
Bamboogle & Multi-hop QA & 125 & Test \\
\bottomrule
\end{tabular}
\caption{Statistics of the training and test datasets used by all methods.}
\label{tab:datasets_statics}
\end{table}

\subsection{Selection of Compression Granularity}
\label{sec:granularity_selection}

To establish a principled framework for compression layer division, we analyzed 1,949 samples of successfully resolved real-world dialogues. Utilizing GPT-4, we generated compressed summaries for each dialogue turn and performed K-means clustering on the summary lengths of these samples, exploring cluster counts ($K$) ranging from 2 to 8. As illustrated in Table~\ref{tab:db_index_comparison}, the Davies-Bouldin Index reaches its minimum value of 0.4988 at $K=5$, indicating that a five-cluster configuration yields the optimal separation quality.

\begin{table}[h]
    \centering
    \footnotesize
    \setlength{\tabcolsep}{3pt}
    \caption{Davies-Bouldin Index Comparison Across Different K Values.}
    \label{tab:db_index_comparison}
    \begin{tabular}{@{}lccccccc@{}}
    \toprule
    \textbf{K} & \textbf{2} & \textbf{3} & \textbf{4} & \textbf{5}$^{\star}$ & \textbf{6} & \textbf{7} & \textbf{8} \\
    \midrule
    \textbf{DB Index} & 0.616 & 0.570 & 0.512 & \textbf{0.499} & 0.506 & 0.521 & 0.513 \\
    \textbf{$\Delta$\%} & --- & $-7.5$ & $-10.2$ & \textbf{$-2.5$} & $+1.4$ & $+3.0$ & $-1.5$ \\
    \bottomrule
    \end{tabular}
\end{table}

Based on these clustering results, we formally defined five compression levels:
\begin{itemize}
    \setlength\itemsep{0.1em}
    \item \textit{Level 1 (100--500 chars):} Ultra-coarse compression, retaining only core conclusions.
    \item \textit{Level 2 (500--1000 chars):} Coarse compression, retaining primary questions and key conclusions.
    \item \textit{Level 3 (1000--2000 chars):} Medium compression, retaining decision points and important context.
    \item \textit{Level 4 (2000--3000 chars):} Fine compression, retaining the complete logical chain and important details.
    \item \textit{Level 5 (3000--6000 chars):} Ultra-fine compression, preserving nearly all critical information.
\end{itemize}

\section{Scaling to 32B Parameters}
\label{appendix:32b_scaling}

To empirically prove that our adaptive multi-granularity compression method does not encounter a performance ceiling at the 14B scale, we extended our framework to a 32B parameter base model (ZipRL-32B). Evaluated across the benchmarks, ZipRL-32B achieves an average EM of 35.2\%, representing a +4.1\% absolute EM improvement over the ZipRL-14B model, demonstrating that ZipRL's performance continues to scale effectively with larger base models.

\section{Same-Base-Model Comparison (Qwen2.5)}
\label{appendix:qwen25_comparison}

To eliminate potential confounds arising from base model generation differences
(e.g., ZipRL-8B uses Qwen3-8B whereas ASearcher-7B uses Qwen2.5-7B), we conduct
strictly matched-scale controlled experiments by instantiating ZipRL on the 
identical Qwen2.5 series (3B/7B/14B) used by the baselines. As shown in 
Table~\ref{tab:scale_matched}, ZipRL consistently outperforms same-series 
baselines at all scales, with a $+5.5$ average EM gain at the 7B scale, firmly 
attributing the observed improvements to our algorithmic contributions.

Table~\ref{tab:scale_matched} presents the strictly matched-scale comparison where ZipRL is trained on the same Qwen2.5 base model series as the baselines, confirming that performance gains stem from our algorithmic contributions rather than generational model differences.

\begin{table}[h]
\centering
\footnotesize
\setlength{\tabcolsep}{3pt}
\caption{Strictly matched-scale comparison: ZipRL (Qwen2.5) vs.\ same-series baselines. $\Delta$EM: absolute improvement over the best same-scale specialized baseline.}
\label{tab:scale_matched}
\begin{tabular}{@{}llccc@{}}
\toprule
\textbf{Scale} & \textbf{Model} & \textbf{Base} & \textbf{Avg EM} & \textbf{$\Delta$EM} \\
\midrule
\multirow{2}{*}{3B}
 & Search-R1-3B     & Qwen2.5-3B  & 8.6  & --- \\
 & ZipRL-3B         & Qwen2.5-3B  & 26.6 & $+18.0$ \\
\midrule
\multirow{2}{*}{7B}
 & ASearcher-7B     & Qwen2.5-7B  & 22.5 & --- \\
 & ZipRL-7B         & Qwen2.5-7B  & 28.0 & $+5.5$ \\
\midrule
\multirow{2}{*}{14B}
 & ASearcher-14B    & Qwen2.5-14B & 27.8 & --- \\
 & ZipRL-14B        & Qwen2.5-14B & 29.6 & $+1.8$ \\
\bottomrule
\end{tabular}
\end{table}

\section{Comparison with Memory-Based Methods}
\label{appendix:memagent_comparison}

MemAgent~\citep{yu2025memagent} and MEM1~\citep{zhou2025mem1} are 
designed for processing pre-given long documents via fixed-length 
memory overwrite or consolidation. To enable fair comparison in our 
multi-turn retrieval setting, we adapt both methods by merging 
retrieved documents and splitting them into 500-word chunks as tool 
observations, injected into their respective memory mechanisms. 
As shown in Table~\ref{tab:memagent_comparison}, ZipRL consistently 
outperforms both adapted baselines across all scales and datasets.

\begin{table*}[h]
\centering
\caption{Comparison with adapted memory-based methods on multi-turn retrieval QA (4 datasets, BC-Plus excluded).}
\label{tab:memagent_comparison}
\setlength{\tabcolsep}{4pt}
\begin{tabular}{@{}lcccccc@{}}
\toprule
\textbf{Scale} & \textbf{Model} & \textbf{Bamboogle} & \textbf{MusiQue} & \textbf{SQuAD} & \textbf{Frames} & \textbf{Avg (EM/F1)} \\
\midrule
\multirow{3}{*}{8B}
 & MemAgent-8B & 35.2/48.9 & 17.8/28.4 & 14.2/26.1 & 13.1/25.2 & 20.1/32.2 \\
 & MEM1-8B     & 38.4/52.4 & 18.8/28.0 & 15.2/28.3 & 13.6/25.8 & 21.5/33.6 \\
 & ZipRL-8B    & \textbf{49.8/64.7} & \textbf{36.8/46.7} & \textbf{25.8/42.4} & \textbf{18.1/29.3} & \textbf{32.6/45.8} \\
\midrule
\multirow{3}{*}{14B}
 & MemAgent-14B & 36.8/50.4 & 16.4/27.2 & 13.4/25.9 & 12.6/24.7 & 19.8/32.1 \\
 & MEM1-14B     & 44.0/60.7 & 25.8/34.0 & 18.2/33.8 & 14.9/27.3 & 25.7/38.9 \\
 & ZipRL-14B    & \textbf{52.8/67.6} & \textbf{33.6/43.4} & \textbf{24.0/42.2} & \textbf{20.8/32.4} & \textbf{32.8/46.4} \\
\bottomrule
\end{tabular}
\end{table*}

\section{Hyperparameter and Sensitivity Analysis}
\label{appendix:hyperparameter_analysis}

\subsection{Hyperparameter Analysis}

To analyze the sensitivity of the hyperparameter $w$ in \cref{equ:advantage_adjust}, we evaluate ZipRL with the Qwen3-8B base model at different values of $w$, as shown in Figure~\ref{fig:token_efficiency_and_w} (right). When $ w = 0 $, the model reverts to standard GRPO, lacking turn-level dense advantage, and performs the worst, validating the effectiveness of HRR. For $ w $ values between 0.2 and 0.8, ZipRL shows consistent strong performance, with the best results at $ w = 0.2 $, which we adopt as the default.

\begin{figure}[h]
    \centering
    \begin{minipage}{0.48\columnwidth}
        \centering
        \includegraphics[width=\textwidth]{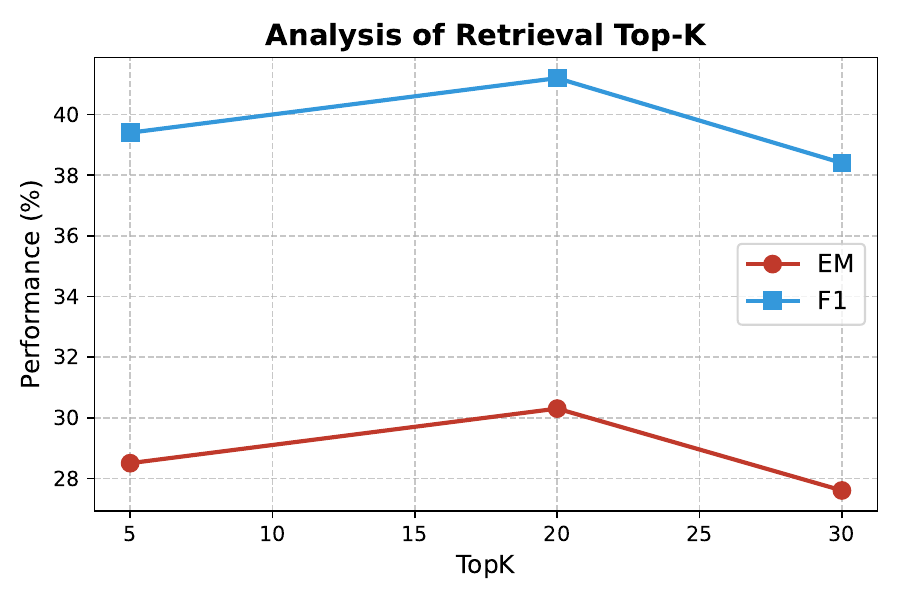}
    \end{minipage}
    \begin{minipage}{0.48\columnwidth}
        \centering
        \includegraphics[width=\textwidth]{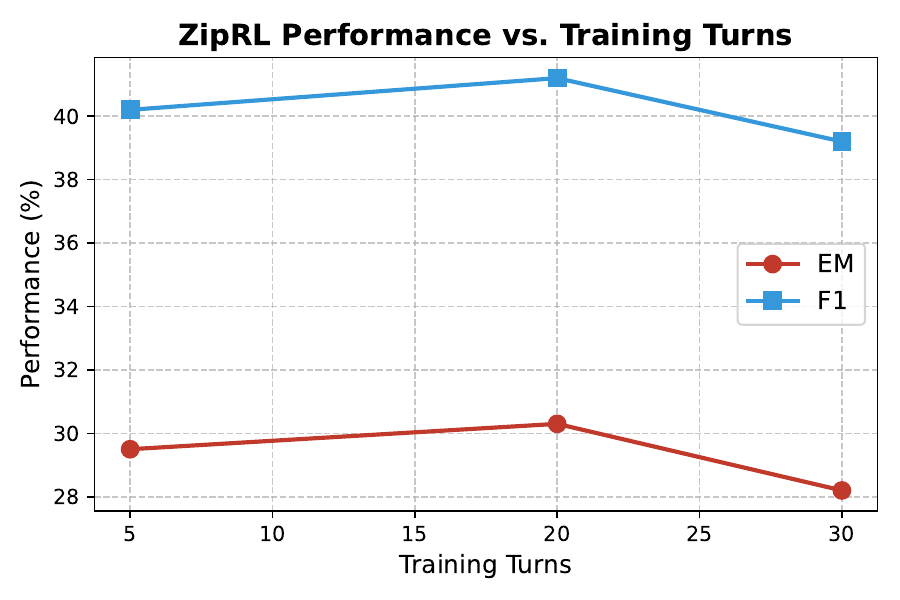}
    \end{minipage}
    \caption{Impact of retrieval top-$k$ (left) and training turns (right) on average EM and F1 scores across five datasets.}
    \label{fig:topk_and_turn_analysis}
\end{figure}

\subsection{Retrieval Top-K Analysis}

To assess the impact of the number of retrieved documents on ZipRL's performance, we evaluate the model at top-k values of 5, 20, and 30. As shown in Figure~\ref{fig:topk_and_turn_analysis} (left), the best performance occurs when top-k is set to 20, which we adopt as the default parameter.

\subsection{Training Turns Analysis}

To investigate the impact of the maximum number of interaction turns on ZipRL's performance, we evaluate the model at maximum training turns of 5, 20, and 30. As shown in Figure~\ref{fig:topk_and_turn_analysis} (right), the model achieves the best performance when the maximum number of turns is set to 20, which we select as the default parameter.

\section{Agent Behavior Analysis}
\label{appendix:agent_behavior_analysis}

\begin{table}[h]
\centering
\footnotesize
\setlength{\tabcolsep}{4pt}
\caption{Fine-grained agent behavior statistics (averaged across MusiQue, SQuAD, Frames, Bamboogle).}
\label{tab:behavior_stats}
\begin{tabular}{@{}lcccc@{}}
\toprule
\textbf{Model} & \textbf{Avg} & \textbf{Avg} & \textbf{Avg} & \textbf{Finish} \\
               & \textbf{Turns} & \textbf{Search/q} & \textbf{Open/q} & \textbf{Rate} \\
\midrule
ZipRL (Ours)  & 7.07 & 5.83 & 0.19 & 0.85 \\
AgentFold     & 6.23 & 4.94 & 0.54 & 0.58 \\
ASearcher     & 4.02 & 3.02 & 0.00 & 0.98 \\
NestBrowse-8B & 3.51 & 1.35 & 0.00 & 0.83 \\
Search-R1-7B  & 8.90 & 8.09 & 0.07 & 0.20 \\
\bottomrule
\end{tabular}
\end{table}

We highlight four key observations. \textbf{(1) Adaptive difficulty-aware search depth:} ZipRL dynamically adjusts retrieval intensity based on question difficulty: it issues only 2.5 searches on Bamboogle but scales to 8.5 on Frames ($3.4\times$ increase), an emergent behavior arising from outcome-driven RL training. \textbf{(2) Excessive retrieval without convergence (Search-R1):} Search-R1-7B issues $\sim$9.6 searches on Bamboogle ($3.8\times$ more than ZipRL) yet achieves significantly lower accuracy, with a Finish Rate of only 0.20, indicating frequent context budget exhaustion without conclusive answers. \textbf{(3) Under-retrieval is not efficiency:} ReAct and NestBrowse models issue very few searches (0.87--1.26 per query), but this reflects under-retrieval rather than genuine efficiency; these models underperform substantially on multi-hop datasets requiring evidence chaining. \textbf{(4) Reliable termination:} ZipRL achieves a Finish Rate of 0.85, substantially outperforming AgentFold (0.58) and Search-R1-7B (0.20), demonstrating that ZipRL reliably commits to answers after sufficient evidence collection.

\begin{table}[h]
\centering
\footnotesize
\setlength{\tabcolsep}{2.5pt}
\caption{Average actual interaction turns and performance on the BC-Plus deep research benchmark.}
\label{tab:bc_plus_stats}
\begin{tabular}{@{}lccc@{}}
\toprule
\textbf{Model} & \textbf{Avg Turns} & \textbf{Average EM} & \textbf{Average F1} \\
\midrule
ZipRL-8B (Ours)     & 13.8 & 0.3233 & 0.4082 \\
GPT-4o-ReAct        & 6.0  & 0.2520 & 0.3161 \\
DeepSeek-v3.2-ReAct & 3.1  & 0.0880 & 0.1329 \\
Qwen3-235B-ReAct    & 2.1  & 0.2067 & 0.2355 \\
NestBrowse-8B       & 1.3  & 0.0880 & 0.1453 \\
\bottomrule
\end{tabular}
\end{table}

To demonstrate deep interaction capabilities on the BrowseComp-Plus (BC-Plus) benchmark, which requires complex deep research, ZipRL maintains effective exploration for an average of 13.8 turns, while baselines like Qwen3-235B-ReAct and Search-R1 terminate prematurely (2.1 to 6.0 turns) due to under-retrieval or context overload (Table~\ref{tab:bc_plus_stats}).

\textbf{Reward Evolution.} As shown in Figure~\ref{fig:reward_and_compression} (left), ZipRL outperforms the baseline in both convergence speed and final reward. GRPO exhibits instability due to sparse supervision, while HRR mitigates the credit assignment challenge in long-horizon tasks. By introducing dense intermediate feedback, HRR enables stable optimization, allowing the policy to reach higher reward levels.

\textbf{Compression Quality.} As illustrated in Figure~\ref{fig:reward_and_compression} (right), ZipRL exhibits a consistent upward trajectory with reduced variance, improving from 0.57 to 0.64.

\paragraph{Multi-Run Stability.}
To empirically validate the predicted variance reduction effect, we conduct 3 independent training runs for both standard GRPO ($w=0$) and ZipRL with HRR. As shown in Table~\ref{tab:stability}, HRR consistently achieves lower cross-run standard deviation ($\sigma = 0.0099$--$0.0124$) compared to GRPO ($\sigma = 0.0351$--$0.0638$) across all training phases.

\begin{table}[h]
\centering
\setlength{\tabcolsep}{3.5pt}
\caption{Cross-run reward statistics (mean $\pm$ std) over 3 independent runs.}
\label{tab:stability}
\begin{tabular}{@{}lcc@{}}
\toprule
\textbf{Training Steps} & \textbf{GRPO ($w=0$)} & \textbf{ZipRL (HRR)} \\
\midrule
Steps 21--25 & $0.5926 \pm 0.0351$ & $0.6171 \pm 0.0111$ \\
Steps 26--30 & $0.6028 \pm 0.0423$ & $0.6263 \pm 0.0113$ \\
Steps 31--35 & $0.6197 \pm 0.0512$ & $0.6379 \pm 0.0124$ \\
Steps 36--40 & $0.6238 \pm 0.0638$ & $0.6439 \pm 0.0099$ \\
\bottomrule
\end{tabular}
\end{table}

\section{Trajectory Distribution Analysis}
\label{appendix:trajectory_analysis}

\begin{figure}[t]
    \centering
    \includegraphics[width=0.9\columnwidth]{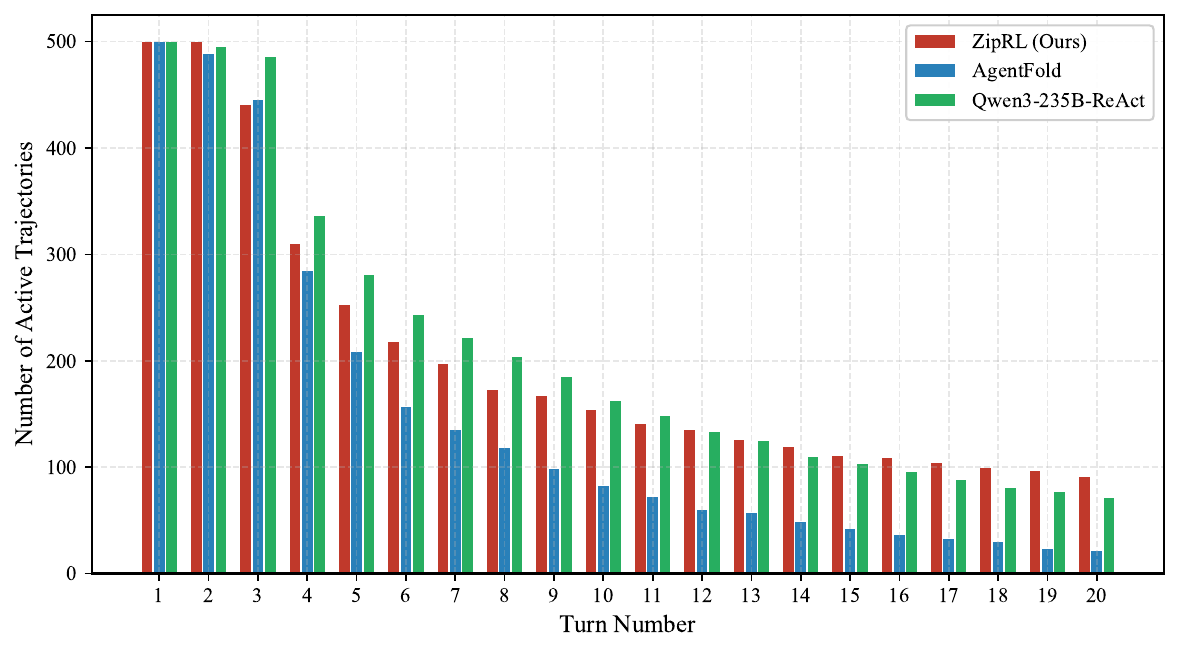}
    \caption{Distribution of active trajectories across different interaction turns on 500 difficulty queries from the MusiQue dataset. Sharp drops in baseline models between turns 5--10 indicate premature termination, while ZipRL maintains consistent reasoning depth up to 20 turns.}
    \label{fig:trajectory}
\end{figure}

As shown in Figure~\ref{fig:trajectory}, the baseline models exhibit sharp trajectory interruptions between turns 5 and 10, but the failure mechanisms differ. \textbf{Context Overload (Qwen3-235B-ReAct):} Despite its vast parameter scale, Qwen3-235B-ReAct achieves an EM of only $18.2\%$ under a full history retention strategy. \textbf{Cognitive Dead-ends (AgentFold):} AgentFold's fixed compression strategy lacks dynamic perception of task relevance, and aggressive compression often prunes ``bridge entities'' essential for multi-hop reasoning. In contrast, ZipRL demonstrates significant reasoning maintenance in deep interaction intervals (10--20 turns).

\section{Pseudocode of ZipRL}
\label{appendix:pseudocode}

Algorithm~\ref{alg:ziprl_hrr} presents the complete training procedure of ZipRL, covering both the cold-start SFT phase and the HRR-based RL optimization phase.

\begin{algorithm*}[t]
\caption{ZipRL: Adaptive Multi-Turn Context Compression with Hindsight Response Replay}
\label{alg:ziprl_hrr}
\begin{algorithmic}[1]
\STATE \textbf{Input:} Training Query Set $\mathcal{Q}$, SFT Model $\pi_{\text{sft}}$, Reference Model $\pi_{\text{ref}}$, Group Size $G$, Reshaping Weight $w$.
\STATE \textbf{Output:} Optimized Policy $\pi_{\theta}$.
\STATE \textbf{Stage 1: Supervised Fine-Tuning (SFT)}
\STATE Initialize $\theta \leftarrow \text{Train}(\pi_{\text{init}}, \mathcal{D}_{\text{demo}})$
\STATE \textbf{Stage 2: Reinforcement Learning with HRR}
\WHILE{not converged}
    \STATE Sample a batch of queries $q \sim \mathcal{Q}$
    \FOR{each query $q$}
        \STATE Sample $G$ trajectories $\mathcal{O} = \{o^{(1)}, \dots, o^{(G)}\}$ using $\pi_{\theta}(\cdot|q)$
        \STATE Obtain sparse outcome rewards $\mathcal{R}_{\text{final}} = \{R_{\text{final}}^{(1)}, \dots, R_{\text{final}}^{(G)}\}$
    \ENDFOR
    \FOR{each trajectory $o^{(i)}$ in $\mathcal{O}$}
        \FOR{each turn $j = 1$ to $T$}
            \STATE Calculate \textbf{Turn Quality}: $Q_{\text{com}}^{(i,j)} \leftarrow \text{Score}(o^{(i,j)})$
        \ENDFOR
        \STATE Calculate \textbf{Trajectory Baseline}: $\bar{Q}_{\text{com}}^{(i)} \leftarrow \frac{1}{T}\sum_{j} Q_{\text{com}}^{(i,j)}$
    \ENDFOR
    \STATE Compute group stats: $\mu_{R} \leftarrow \text{mean}(\mathcal{R}_{\text{final}}), \sigma_{R} \leftarrow \text{std}(\mathcal{R}_{\text{final}})$
    \FOR{each trajectory $i$ and turn $j$}
        \STATE $A_{\text{GRPO}}^{(i)} \leftarrow (R_{\text{final}}^{(i)} - \mu_{R}) / (\sigma_{R} + \varepsilon)$
        \STATE $A_{\text{HRR}}^{(i,j)} \leftarrow A_{\text{GRPO}}^{(i)} + w \cdot (Q_{\text{com}}^{(i,j)} - \bar{Q}_{\text{com}}^{(i)})$
    \ENDFOR
    \STATE Update $\theta$ by maximizing $\mathcal{L}(\theta)$
\ENDWHILE
\end{algorithmic}
\end{algorithm*}

\section{Robustness Under Noisy Retrieval}
\label{appendix:noise_robustness}

We evaluate whether ZipRL adaptively adjusts its compression strategy when retrieval quality degrades. Using ZipRL-8B, we inject noise into the retrieved document pool at five levels: Clean (0\%), Low (25\%), Mid (50\%), High (75\%), and Adversarial (100\% adversarial documents).

\paragraph{Robustness to Random Noise (0--75\%).}
Performance degrades gracefully under random noise. Across four datasets, the EM drop from Clean to High (75\%) noise is modest: $-9.8\%$ on Bamboogle, $-4.2\%$ on MusiQue, $-11.2\%$ on SQuAD, and $-15.3\%$ on Frames. ZipRL retains approximately 85--90\% of clean performance at 75\% noise.

\paragraph{Adversarial Retrieval Boundary.}
Under fully adversarial conditions (100\% adversarial documents), performance degrades substantially across all datasets (e.g., MusiQue: $0.330 \to 0.002$), consistent with the shared assumption of retrieval trustworthiness in RAG-based systems. Notably, ZipRL's average retrieval turns increase from $\sim$8.6 to $\sim$35.9 (Table~\ref{tab:noise_robustness}), suggesting emergent contradiction-resolution behavior. Extending the compression scoring framework with retrieval credibility estimation remains future work.

\begin{table}[h]
\centering
\footnotesize
\setlength{\tabcolsep}{2.5pt}
\caption{ZipRL-8B performance and average retrieval turns under varying noise levels (EM / Avg Turns).}
\label{tab:noise_robustness}
\begin{tabular}{@{}lcccc@{}}
\toprule
\textbf{Noise Level} & \textbf{Bamboogle} & \textbf{MusiQue} & \textbf{SQuAD} & \textbf{Frames} \\
 & EM / Turns & EM / Turns & EM / Turns & EM / Turns \\
\midrule
Clean (0\%)       & 0.488 / 8.6  & 0.330 / 15.4 & 0.232 / 6.4  & 0.163 / 20.7 \\
Low (25\%)        & 0.440 / 8.1  & 0.324 / 15.5 & 0.256 / 6.5  & 0.160 / 21.0 \\
Mid (50\%)        & 0.440 / 8.4  & 0.312 / 15.3 & 0.218 / 8.4  & 0.160 / 22.4 \\
High (75\%)       & 0.440 / 9.8  & 0.316 / 17.0 & 0.206 / 9.2  & 0.138 / 23.6 \\
Adversarial       & 0.072 / 35.9 & 0.002 / 36.7 & 0.012 / 31.6 & 0.018 / 37.7 \\
\bottomrule
\end{tabular}
\end{table}

\section{Comparison with Adaptive RAG Compression Methods}
\label{appendix:rag_compression_comparison}

\begin{table}[h]
\centering
\footnotesize
\setlength{\tabcolsep}{1pt}
\caption{Taxonomy of context compression paradigms. \cmark: supported; \xmark: not supported.}
\label{tab:method_taxonomy}
\begin{tabular}{@{}lcccc@{}}
\toprule
\textbf{Method} & \textbf{Multi-} & \textbf{Per-doc} & \textbf{RL-} & \textbf{Joint} \\
                & \textbf{turn}   & \textbf{Adaptive} & \textbf{optimized} & \textbf{Retrieval} \\
\midrule
LLMLingua / RECOMP    & \xmark & \xmark & \xmark & \xmark \\
ACC-RAG / SARA        & \xmark & \cmark & \xmark & \xmark \\
AttnComp              & \xmark & \cmark & \xmark & \xmark \\
MemAgent / MEM1       & \cmark & \xmark & \xmark & \xmark \\
\textbf{ZipRL (ours)} & \cmark & \cmark & \cmark & \cmark \\
\bottomrule
\end{tabular}
\end{table}

Recent adaptive RAG compression methods such as ACC-RAG~\citep{guo2025enhancing}, AttnComp~\citep{luo2025attncomp}, and SARA~\citep{jin2025sara} pursue variable-rate context compression in single-turn settings, with no mechanism to issue new retrieval actions or receive optimization signal from downstream task outcomes via RL. ZipRL addresses a fundamentally different problem: \emph{online, multi-turn agentic search}, where compression and agentic behavior are co-optimized through sparse end-task reward.

\section{Robustness of Compression Quality Score Against Gaming Behaviors}
\label{appendix:gaming_robustness}

\textbf{1. Mutual Structural Constraints.} While generating redundant query keywords might marginally increase $Q_{\mathrm{info}}$, this behavior is strictly penalized by $Q_{\mathrm{ratio}}$ (length overflow) and $Q_{\mathrm{sem}}$ (destroyed grammatical coherence). Since the final $Q_{\mathrm{com}}$ is a weighted sum, the penalty in structural and semantic dimensions heavily outweighs the illicit gain in $Q_{\mathrm{info}}$.

\textbf{2. Mathematical Bounding.} The keyword coverage term $S_{\mathrm{key}} = |K_q \cap y| / (|K_q \cap x| + \epsilon)$ naturally caps at 1.0, preventing unbounded reward through keyword repetition.

\textbf{3. Anchoring by Final Task Reward.} HRR functions as an advantage reshaping technique rather than replacing the environment reward. The final advantage remains fundamentally anchored by the GRPO advantage, which is strictly determined by the exact match (EM/F1) of the final answer. This is empirically confirmed by the monotonically increasing Pearson correlation between $Q_{\mathrm{com}}$ and task reward ($r: 0.25 \to 0.51$).

\section{Prompts}
\label{appendix:prompt_details}

We used the prompts shown in Figure~\ref{fig:prompt1} and Figure~\ref{fig:prompt2}.

\begin{figure*}[ht]
    \centering
    \includegraphics[width=\linewidth]{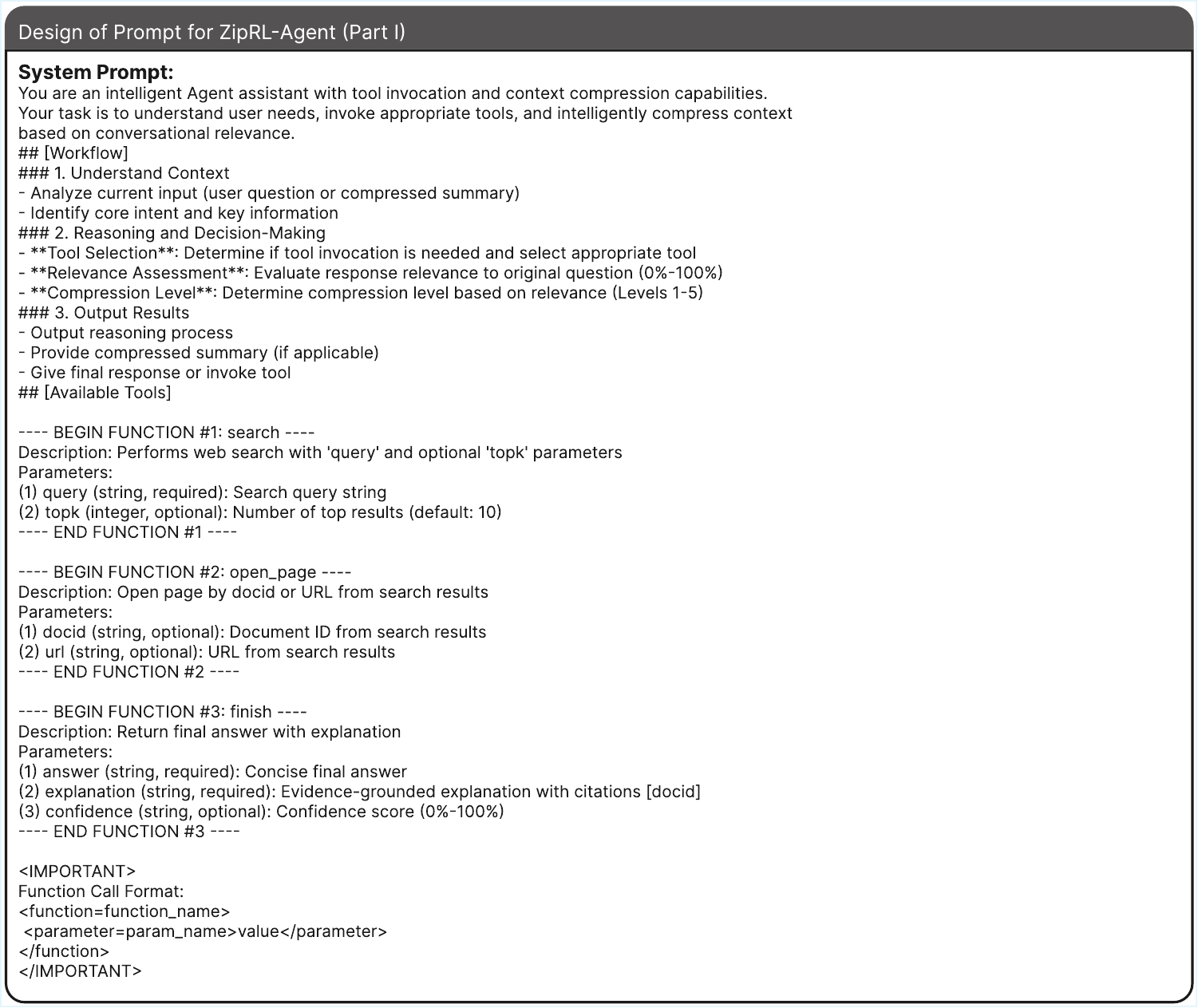}
    \caption{Prompt Design for ZipRL, Part 1.}
    \label{fig:prompt1}
\end{figure*}

\begin{figure*}[ht]
    \centering
    \includegraphics[width=\linewidth]{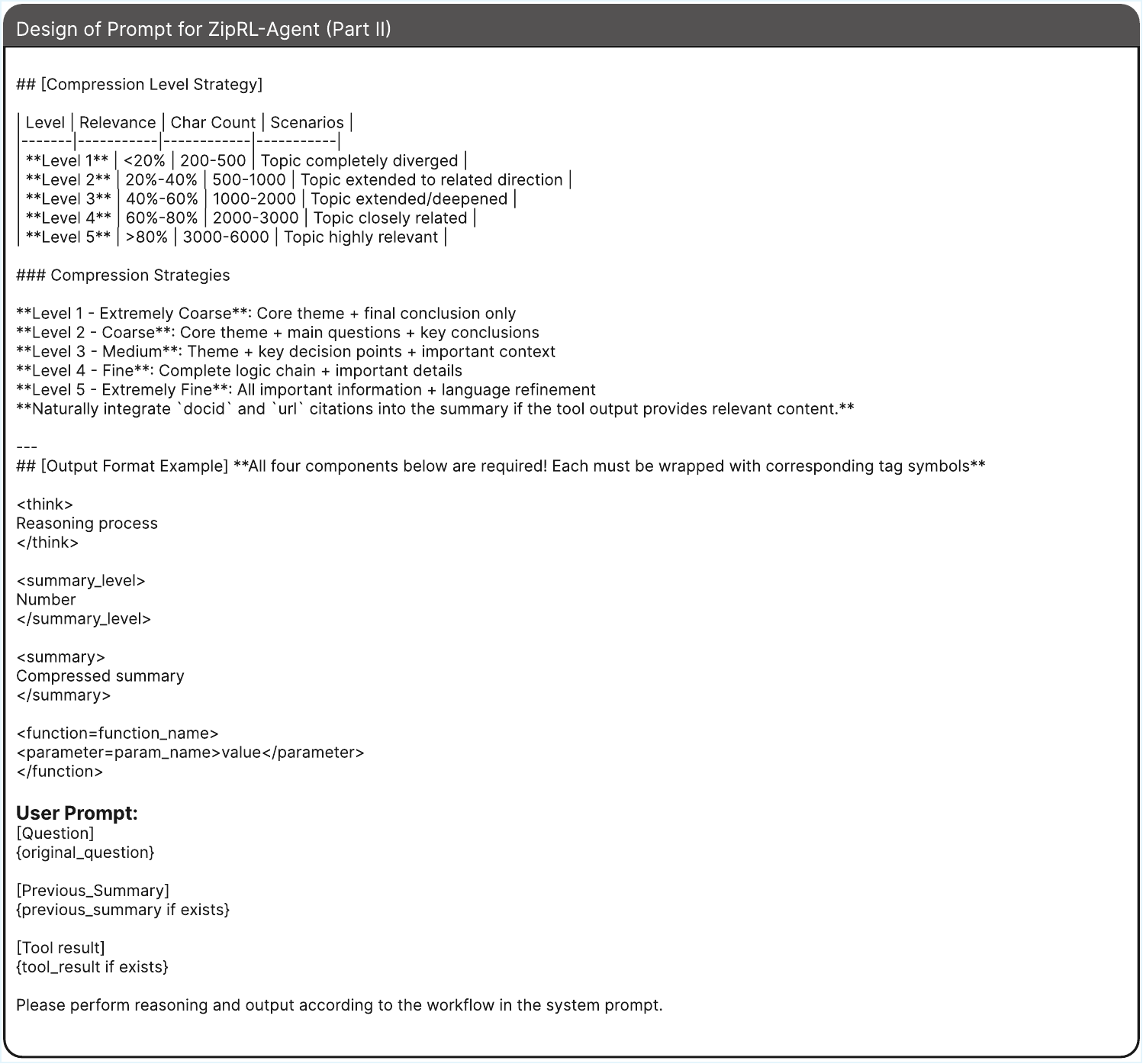}
    \caption{Prompt Design for ZipRL, Part 2.}
    \label{fig:prompt2}
\end{figure*}

\end{document}